\documentclass[acmsmall]{acmart}

\AtBeginDocument{%
  \providecommand\BibTeX{{%
    \normalfont B\kern-0.5em{\scshape i\kern-0.25em b}\kern-0.8em\TeX}}}

\setcopyright{acmcopyright}

\acmJournal{TELO}

\usepackage{hyperref}       
\usepackage{url}            
\usepackage{amsfonts}       

\usepackage{graphicx}
\usepackage{subfigure}
\usepackage{algorithm}
\usepackage{algorithmic}
\usepackage{amssymb}
\usepackage{amsmath}\allowdisplaybreaks
\usepackage{amsthm}
\usepackage{mathtools}
\usepackage{bm}
\usepackage[american]{babel}
\usepackage{enumitem}
\usepackage{float,dblfloatfix}
\usepackage{multirow}
\mathtoolsset{showonlyrefs,showmanualtags}

\newtheorem{theorem}{Theorem}
\newtheorem{lemma}{Lemma}

\newtheorem{assumption}{Assumption}

\newenvironment{myproof}[1]
{\par\noindent\textbf{Proof of #1.}\ \enspace\ignorespaces\begin{allowdisplaybreaks}}
	{\end{allowdisplaybreaks}\hspace{\stretch{1}}$\square$}

\renewcommand{\algorithmicrequire}{\textbf{Input:}} 

\begin{document}

\title{Bayesian Optimization using Pseudo-Points}

\author{Chao Qian}
\email{qianc@lamda.nju.edu.cn}
\affiliation{%
 \institution{Nanjing University}
 \city{Nanjing}
 \postcode{210023}
 \country{China}}

 \author{Hang Xiong}
\affiliation{%
 \institution{University of Science and Technology of China}
 \city{Hefei}
 \postcode{230027}
 \country{China}}

\author{Ke Xue}
\email{xuek@lamda.nju.edu.cn}
\affiliation{%
 \institution{Nanjing University}
 \city{Nanjing}
 \postcode{210023}
 \country{China}}


\begin{abstract}
Bayesian optimization (BO) is a popular approach for expensive black-box optimization, with applications including parameter tuning, experimental design, robotics. BO usually models the objective function by a Gaussian process (GP), and iteratively samples the next data point by maximizing an acquisition function. In this paper, we propose a new general framework for BO by generating pseudo-points (i.e., data points whose objective values are not evaluated) to improve the GP model. With the classic acquisition function, i.e., upper confidence bound (UCB), we prove that the cumulative regret can be generally upper bounded. Experiments using UCB and other acquisition functions, i.e., probability of improvement (PI) and expectation of improvement (EI), on synthetic as well as real-world problems clearly show the advantage of generating pseudo-points.
\end{abstract}



\maketitle

\section{Introduction}

In real-world applications, we often need to solve an optimization problem: $$\bm{x}^*\in \mathop{\arg\max}\nolimits_{\bm{x}\in \mathcal{X}}f(\bm{x}),$$ where $\mathcal{X} \subseteq \mathbb{R}^d$ is the solution space, $f:\mathcal{X} \rightarrow \mathbb{R}$ is the objective function, and $\bm{x}^*$ is a global optimal solution. Usually, it is assumed that $f$ has a known mathematical expression, is convex, or cheap to evaluate at least. Increasing evidences, however, show that $f$ may not satisfy these assumptions, but is an expensive black-box model~\cite{Brochu2010physical}. That is, $f$ can be non-convex, or even the closed-form expression of $f$ is unknown; meanwhile, evaluating $f$ can be noisy and computationally very expensive.

Expensive black-box optimization is involved in many real-world decision making problems. For example, in machine learning, one has to tune hyper-parameters to maximize the performance of a learning algorithm~\cite{Snoek2012PBO}; in physical experiments, one needs to set proper parameters of the experimental environment to obtain an ideal product~\cite{Brochu2010physical}. More applications can been found in robotic control~\cite{MartinezCantin07RSS}, computer vision~\cite{Deni2012cv}, sensor placing~\cite{Garnett2010sensor}, and analog circuit design~\cite{lyu18a}.

Bayesian optimization (BO)~\cite{Jonas1994ApplicationBO} has been a type of powerful algorithm to solve expensive black-box optimization problems. The main idea is to build a model, usually by a Gaussian process (GP), for the objective function $f$ based on the observation data, and then sample the next data point by maximizing an acquisition function. Many BO algorithms have been proposed, with the goal of reaching the global optima using as few objective evaluations as possible.

Most existing works focus on designing effective acquisition functions, e.g., probability of improvement (PI)~\cite{Krushner64PI}, expectation of improvement (EI)~\cite{Jones1998EI}, and upper confidence bound (UCB)~\cite{Srinivas2010GPUCB}. Recently, Wang et al.~\cite{Wang16EST} proposed the EST function by directly estimating $\bm{x}^*$, which automatically and adaptively trades off exploration and exploitation in PI and UCB. Another major type of acquisition functions is based on information entropy, including entropy search (ES)~\cite{Hennig12ES}, predictive ES~\cite{Hern14PES}, max-value ES~\cite{Wang17MES}, FITBO~\cite{xu2018fast}, etc. As BO is a sequential algorithm, some parallelization techniques have been introduced for acceleration, e.g.,~\cite{Azimi12SimPal,DesautelsKB14PUCB,Shah15Batch,Gonzal16LocalPal}. There is also a sequence of works addressing the difficulty of BO for high-dimensional optimization, e.g.,~\cite{Wang2014REMBO,Kirthevasan2015Additive,wang2017batched,Mutny18FQ}.


For any BO algorithm with a specific acquisition function, the GP model becomes increasingly accurate with the observation data augmenting. However, the number of data points to be evaluated is often limited due to the expensive objective evaluation. In this paper, we propose a general framework for BO by generating pseudo-points to improve the GP model. That is, before maximizing the acquisition function to select the next point in each iteration, some pseudo-points are generated and added to update the GP model. The pseudo-points are neighbors of the observed data points, and take the same function values as the observed ones. Without increasing the evaluation cost, the generation of pseudo-points can reduce the variance of the GP model, while introducing little accuracy loss under the Lipschitz assumption. This framework is briefly called BO-PP.

Theoretically, we study the performance of BO-PP w.r.t. the acquisition function UCB, called UCB-PP. We prove a general upper bound of UCB-PP on the cumulative regret, i.e., $\sum^T_{t=1} (f(\bm{x}^*) -f(\bm{x}_t))$, where $\bm{x}_t$ denotes the sampled point in the $t$-th iteration. It is shown to be a generalization of the known bound~\cite{Srinivas2010GPUCB} of UCB.

Empirically, we compare BO-PP with BO on synthetic benchmark functions as well as real-world optimization problems. The acquisition functions UCB, PI and EI are selected. The results clearly show the excellent performance of BO-PP. The superiority of UCB-PP over UCB verifies our theoretical analysis, and that of PI-PP over PI and EI-PP over EI shows the applicability of the proposed framework of generating pseudo-points.

We start the rest of the paper by introducing some background on BO. We then present in three subsequent sections the proposed framework BO-PP, theoretical analysis and empirical study, respectively. The final section concludes the paper.

\section{Background}

The general framework of BO is shown in Algorithm~\ref{alg:SBO}. It sequentially optimizes some given objective function $f(\bm{x})$ with assumptions on a prior distribution, i.e., a probabilistic model, over $f(\bm{x})$. In each iteration, BO selects a point $\bm{x}$ by maximizing an acquisition function $acq(\cdot)$, evaluates its objective value $f(\bm{x})$, and updates the prior distribution with the new data point.

\subsection{GPs}

A GP~\cite{Rasmussen2005GP} is commonly used as the prior distribution, which regards the $f$ value at each data point as a random variable, and assumes that all of them satisfy a joint Gaussian distribution specified by the mean value function $m(\cdot)$ and the covariance function $k(\cdot,\cdot)$. For convenience, $m(\cdot)$ is set to zero. Assume that the objective evaluation is subject to i.i.d. additive Gaussian noise, i.e., $y=f(\bm{x})+\epsilon$, where $\epsilon \sim \mathcal{N}(0,\sigma^2)$. Let $[t]$ denote the set $\{1,2,\ldots,t\}$.

Given an observation data $D_t=\{(\bm{x}_i,y_i)\}^t_{i=1}$, we can obtain the posterior mean
\begin{align}\label{eq-mean}
\mu_t(\bm{x})=\bm{k}_t(\bm{x})^\mathrm{T}(\mathbf{K}_{t}+\sigma^2\mathbf{I})^{-1}\bm{y}_{1:t}
\end{align} and the posterior variance
\begin{align}\label{eq-variance}
\sigma^2_t(\bm{x})=k(\bm{x},\bm{x})-\bm{k}_t(\bm{x})^\mathrm{T}(\mathbf{K}_{t}+\sigma^2\mathbf{I})^{-1}\bm{k}_t(\bm{x}),
\end{align} where $\bm{k}_t(\bm{x})=[k(\bm{x}_i,\bm{x})]_{i=1}^t$, $\mathbf{K}_{t}=[k(\bm{x}_i,\bm{x}_j)]_{i,j \in [t]}$ and $\bm{y}_{1:t}=[y_1;y_2;\ldots;y_t]$. For a GP, the log likelihood of observed data $D_t$ is $$\log \mathrm{Pr}(\bm{y}_{1:t}\mid \{\bm{x}_i\}^t_{i=1},\bm{\theta})=-\frac{1}{2}\bm{y}_{1:t}^\mathrm{T}(\mathbf{K}_t+\sigma ^2\mathbf{I})^{-1}\bm{y}_{1:t}-\frac{1}{2}\log \mathrm{det}(\mathbf{K}_t+\sigma ^2\mathbf{I})-\frac{t}{2}\log 2\pi,$$ where $\bm{\theta}$ denote the hyper-parameters of $k(\cdot,\cdot)$, and $\mathrm{det}(\cdot)$ denotes the determinant of a matrix. When updating the GP model in line~5 of Algorithm~\ref{alg:SBO}, the hyper-parameters $\bm{\theta}$ can be updated by maximizing the log likelihood of the augmented data, or treated to be fully Bayesian.

\begin{algorithm}[t]
	\caption{BO Framework}
	\label{alg:SBO}
	\begin{algorithmic}[1]
        \REQUIRE iteration budget $T$\\
	    \ENSURE
		\STATE let $D_0=\emptyset$;
		\FOR {$t=1:T$}
		\STATE $\bm{x}_{t}=\arg\max_{\bm{x} \in\mathcal{X}}acq(\bm{x})$;
		\STATE evaluate $f$ at $\bm{x}_{t}$ to obtain $y_t$;
		\STATE augment the data $D_{t}=D_{t-1} \cup \{(\bm{x}_{t},y_{t}) \}$ and update the GP model
		\ENDFOR
	\end{algorithmic}
\end{algorithm}

\subsection{Acquisition Functions}

The data point to be evaluated in each iteration is selected by maximizing an acquisition function, which needs to trade off exploration, i.e., large posterior variances, and exploitation, i.e., large posterior means. Many acquisition functions have been proposed, and we introduce three typical ones, i.e., PI~\cite{Krushner64PI}, EI~\cite{Jones1998EI} and UCB~\cite{Srinivas2010GPUCB}, which will be examined in this paper.

Let $\bm{x}^+$ be the best point generated in the first $(t-1)$ iterations, and $Z=(\mu_{t-1}(\bm{x})-f(\bm{x}^+))/\sigma_{t-1}(\bm{x})$. Let $\Phi$ and $\phi$ denote the cumulative distribution and probability density functions of standard Gaussian distribution, respectively. PI selects the point by maximizing the probability of improvement, i.e.,
\begin{align}\label{eq-PI}
\mathrm{PI}(\bm{x})=\mathrm{Pr}(f(\bm{x})>f(\bm{x}^+))=\Phi(Z).
\end{align}
EI selects the data point by maximizing the expectation of improvement, i.e.,
\begin{equation}\label{eq-EI}
\mathrm{EI}(\bm{x})=\left\{
\begin{aligned}
&(\mu_{t-1}(\bm{x})-f(\bm{x}^+))\Phi(Z)+\sigma_{t-1}(\bm{x})\phi(Z) &\text{if}\ \sigma_{t-1}(\bm{x})>0,\\
&0  &\text{if}\ \sigma_{t-1}(\bm{x})=0.
\end{aligned}
\right.
\end{equation}
UCB integrates the posterior mean and variance via a trade-off parameter $\beta_t$, i.e.,
\begin{align}\label{eq-UCB}
\mathrm{UCB}(\bm{x})=\mu_{t-1}(\bm{x})+\beta_t^{1/2}\sigma_{t-1}(\bm{x}),
\end{align}
and selects the data point by maximizing this measure.

\subsection{Regrets}

To evaluate the performance of BO algorithms, regrets are often used. The instantaneous regret $r_t=f(\bm{x}^*)-f(\bm{x}_t)$ measures the gap of function values between a global optimal solution $\bm{x}^*$ and the currently selected point $\bm{x}_t$. The simple regret $S_T=\min_{i\in[T]} r_i$ measures the gap between $\bm{x}^*$ and the best point found in the first $T$ iterations. The cumulative regret $R_T=\sum_{i=1}^{T}r_i$ is the sum of instantaneous regrets in the first $T$ iterations. It is clear that the simple regret $S_T$ is upper bounded by the average of the cumulative regret, i.e., $R_T/T$.

\section{The BO-PP Framework}

In BO, a GP is used to characterize the unknown objective function. The posterior variance of a GP describes the uncertainty about the unknown objective, while the posterior mean provides a closed form of the unknown objective. As the observation data augments, the posterior variance decreases and the posterior mean gets close to the unknown objective, making the GP express the unknown objective better. Thus, a straightforward way to improve the GP model is collecting more data points, which is, however, impractical, because the objective evaluation is expensive. In this section, we propose a general framework BO-PP by generating pseudo-points to improve the GP model.

As shown in Eq.~(\refeq{eq-variance}), the posterior variance of $f$ does not depend on the objective values, and will be decreased by adding new data points. As shown in Eq.~(\refeq{eq-mean}), the posterior mean of $f$ can be regarded as a linear combination of the observed objective values, and will be influenced by the error on the objective values of new data points. Inspired by the Lipschitz assumption, i.e., close data points have close objective values, the pseudo-points are selected to be neighbors of the observed data points, and take the same objective values as the observed ones.

The BO-PP framework is described in Algorithm~\ref{alg:GBO}. Before selecting the next data point in line~5, BO-PP generates a few pseudo-points to re-compute the posterior mean and variance of the GP model in lines~3-4, rather than directly using the GP model updated in the last iteration. After evaluating a new data point in line~6, the hyper-parameters of the covariance function employed by the GP model will be updated in line~7 using the truly observed data points by far. Note that the pseudo-points are only used to re-compute the posterior mean and variance.

\begin{algorithm}[t]
	\caption{BO-PP Framework}
	\label{alg:GBO}
	\begin{algorithmic}[1]
	\REQUIRE iteration budget $T$\\
    \renewcommand{\algorithmicrequire}{\textbf{Parameter:}}
	\REQUIRE $\{l_i\}^{T-1}_{i=0}$, $\{\tau_i\}^{T-1}_{i=0}$\\
	\ENSURE
		\STATE let $D_0=\emptyset$, $l_0=0$ and $\tau_0=0$;
		\FOR {$t=1:T$}
		\STATE generate $l_{t-1}$ pseudo-points $\{(\bm{x}'_i,\hat y'_i)\}^{l_{t-1}}_{i=1}$;
		\STATE re-compute $\hat{\mu}_{t-1}$ and $\hat{\sigma}_{t-1}$ by $D_{t-1} \cup \{(\bm{x}'_i,\hat y'_i)\}^{l_{t-1}}_{i=1}$;
		\STATE $\bm{x}_{t}=\arg\max_{\bm{x} \in\mathcal{X}}acq(\bm{x})$;
		\STATE evaluate $f$ at $\bm{x}_{t}$ to obtain $y_t$;
		\STATE augment the data $D_{t}=D_{t-1} \cup \{(\bm{x}_{t},y_{t}) \}$ and update the GP model
		\ENDFOR
    \renewcommand{\algorithmicrequire}{}
    \REQUIRE where each pseudo-point in the $t$-th iteration has distance $\tau_{t-1}$ to some observed data point in $D_{t-1}$, and takes the same objective value as the observed one.
	\end{algorithmic}
\end{algorithm}

The way of generating pseudo-points can be diverse, e.g., randomly sampling a point with distance $\tau$ from some observed data point.\footnote{Here, two data points $\bm{x},\bm{x}' \in \mathbb{R}^d$ have distance $\tau$ means that $\forall i \in [d]: |x_i-x'_i|=\tau$.} The only requirement is that the pseudo-point takes the same objective value as the corresponding observed data point, which does not increase the evaluation cost. The number $l_t$ of pseudo-points and the distance $\tau_t$ employed in each iteration could affect the performance of the algorithm. For example, as $\tau_t$ decreases, the error on the objective values of pseudo-points will decrease, whereas the reduction on the posterior variance will also decrease. Their relationship will be analyzed in the theoretical analysis. Note that BO-PP can be equipped with any acquisition function.

\section{Theoretical Analysis}\label{sec-theory}

In this section, we theoretically analyze the performance of BO-PP w.r.t. the acquisition function UCB, called UCB-PP. Specifically, we prove that the cumulative regret $R_T$ of UCB-PP can be generally upper bounded.

In the following analysis, let $\mu_t$ and $\sigma_t$ denote the posterior mean and variance after obtaining $D_t$; let $\hat{\mu}_t$ and $\hat{\sigma}_t$ denote the posterior mean and variance after adding pseudo-points $\{(\bm{x}'_i,\hat y'_i)\}^{l_t}_{i=1}$ into $D_t$; let $\tilde{\mu}_t$ and $\tilde{\sigma}_t$ denote the posterior mean and variance after adding pseudo-points with true observed values, i.e., $\{(\bm{x}'_i,y'_i))\}^{l_t}_{i=1}$, where $y'_i=f(\bm{x}'_i)+\epsilon'_i$ with $\epsilon'_i \sim \mathcal{N}(0,\sigma^2)$. Some notations about pseudo-points: $\hat{\bm{y}}'_{1:l_t}=[\hat y'_1;\hat y'_2;\ldots;\hat y'_{l_t}]$; $\bm{y}'_{1:l_t}=[y'_1;y'_2;\ldots;y'_{l_t}]$; $\bm{k}'_{l_t}(\bm{x})=[k(\bm{x}_i',\bm{x})]^{l_t}_{i=1}$; $\mathbf{K}'_{l_t}=[k(\bm{x}'_i,\bm{x}_j')]_{i,j\in[l_t]}$; $\widetilde{\mathbf{K}}_{t,l_t}=[k(\bm{x}_i,\bm{x}_j')]_{i\in[t],j\in[l_t]}$; $\bm{p}(\bm{x})=\widetilde{\mathbf{K}}_{t,l_t}^{\mathrm{T}}(\mathbf{K}_t+\sigma^2 \mathbf{I})^{-1}\bm{k}_t(\bm{x})-\bm{k}'_{l_t}(\bm{x})$; $\mathbf{M}=(\mathbf{K}'_{l_t}-\widetilde{\mathbf{K}}_{t,l_t}^{\mathrm{T}}(\mathbf{K}_t+\sigma^2 \mathbf{I})^{-1}\widetilde{\mathbf{K}}_{t,l_t}+\sigma^2\mathbf{I})^{-1}$.
For convenience of analysis, assume $k(\bm{x},\bm{x})=1$.


Let $A$ be a finite subset of $\mathcal{X}$, $\bm{f}_A$ denote their true objective values (which are actually random variables satisfying the posterior Gaussian distribution over the true objective values), and $\bm{y}_A$ denote the noisy observations. Let $PP$ denote all generated pseudo-points, and $\hat{\bm{y}}_{PP}$ denote their selected objective values. Note that $\hat{\bm{y}}_{PP}$ are random variables, as they are actually the noisy observation of the objective values of $PP$'s neighbor observed points. Let $\gamma'_T=\max_{A: |A|=T} I(\bm{y}_A;\bm{f}_A)-\min_{A: |A|=T, PP} I(\bm{y}_A;\hat{\bm{y}}_{PP})$, where $I(\cdot;\cdot)$ denotes the mutual information.

Theorem~\ref{theo:general} gives an upper bound of UCB-PP on the cumulative regret $R_T$. As the analysis of UCB in~\cite{Srinivas2010GPUCB}, Assumption~\ref{assump:lip} is required, implying
\begin{align}\label{eq-assumption} \mathrm{Pr}(\forall \bm{x},\bm{x}': |f(\bm{x})-f(\bm{x}')|\le L\|\bm{x}-\bm{x}'\|_1)\geq 1-dae^{-(L/b)^2}.\end{align}

\begin{assumption}\label{assump:lip}

Suppose the kernel $k(\cdot,\cdot)$ satisfies the following high probability bound on the derivatives of $f$: for some constants $a,b>0$, $\forall j \in [d]: \mathrm{Pr}(sup_{\bm{x}\in \mathcal{X}}|\partial f/ \partial x_j|>L) \le ae^{-(L/b)^2}$.
\end{assumption}

\begin{theorem}\label{theo:general}
Let $\mathcal{X} \subset [0,r]^d$, $\delta\in (0,1)$, and set $\beta_t$ in Eq.~(\refeq{eq-UCB}) as $\beta_t$$=2\log(2\pi^2t^2/(3\delta))+2d \log(t^2 dbr\sqrt{\log(4da/\delta)})$. Running UCB-PP for $T$ iterations, it holds that
\begin{equation}
\begin{aligned}\label{eq-mid-5}
\mathrm{Pr}\Big(R_T\le\sqrt{C T \beta_T\gamma_T'}+2+2\sum\nolimits_{t=1}^{T}\Delta_m(l_{t-1},\tau_{t-1})\Big)\ge 1-\delta,
\end{aligned}
\end{equation}
where $C=8/\log(1+\sigma^{-2})$, and $\Delta_m(l_t,\tau_t)=l_t^2\sqrt{1+\sigma^{-2}} \Big(bd\tau_t\sqrt{\log(4da/\delta)}/\sigma+ 2 \sqrt{\log \frac{4\sum\nolimits^{T-1}_{t=0} l_t}{\delta}}\Big)$.
\end{theorem}

Before giving the proof of Theorem~\ref{theo:general}, we first give some lemmas that will be used. Lemma~\ref{lemma:variance} gives the reduction on the posterior variance by adding pseudo-points.

\begin{lemma}\label{lemma:variance}
After obtaining $D_t$ in UCB-PP, the reduction on the posterior variance by adding pseudo-points $\{(\bm{x}'_i,\hat y'_i)\}^{l_t}_{i=1}$ is \begin{align}
\Delta_v(\bm{x},l_t,\tau_t)=\sigma_t^2(\bm{x})-\hat{\sigma}^2_{t}(\bm{x})=\bm{p}(\bm{x})^{\mathrm{T}}\mathbf{M}\bm{p}(\bm{x}).
\end{align}
\end{lemma}
\begin{proof}
	By Eq.~(\refeq{eq-variance}) and $k(\bm{x},\bm{x})=1$, we have
	\begin{equation}
	\begin{aligned}
	\hat{\sigma}_{t}^2(x)&=1 - [\bm{k}_t(\bm{x});\bm{k}'_{l_t}(\bm{x})]^{\mathrm{T}}
	\left( \left[
	\begin{matrix}
	\mathbf{K}_t & \widetilde{\mathbf{K}}_{t,l_t}\\
	\widetilde{\mathbf{K}}_{t,l_t}^{\mathrm{T}} &\  \mathbf{K}'_{l_t}
	\end{matrix}
	\right]+\sigma^2\mathbf{I}\right)^{-1}[\bm{k}_t(\bm{x});\bm{k}'_{l_t}(\bm{x})]\\
	&=1-\bm{k}_t(\bm{x})^{\mathrm{T}}(\mathbf{K}_t+\sigma^2\mathbf{I})^{-1}\bm{k}_t(\bm{x})-\bm{p}(\bm{x})^{\mathrm{T}}\mathbf{M}\bm{p}(\bm{x})\\
	&=\sigma_t^2(\bm{x})-\bm{p}(\bm{x})^{\mathrm{T}}\mathbf{M}\bm{p}(\bm{x}),
	\end{aligned}
	\end{equation}
where the second equality is derived by the inverse of a block matrix, and the third one holds by Eq.~(\refeq{eq-variance}). Thus, the lemma holds.
\end{proof}

Lemma~\ref{lemma:ineq} is extracted from Lemma~5.1 in~\cite{Srinivas2010GPUCB}, and will be used in the proof of Lemma~\ref{lemma:mean}.
\begin{lemma}\label{lemma:ineq}
Suppose $X$ is a random variable satisfying the Gaussian distribution $\mathcal{N}(\mu,\sigma^2)$. Then, it holds that
\begin{equation}
\mathrm{Pr}(|X-\mu|>c\sigma)\le e^{-c^2/2},
\end{equation}
where $c>0$.
\end{lemma}

The error on the posterior mean led by the incorrect objective values of pseudo-points can be bounded as follows.

\begin{lemma}\label{lemma:mean}	
After obtaining $D_t$ in UCB-PP, the difference on the posterior mean by adding pseudo-points, i.e., $\{(\bm{x}'_i,\hat y'_i)\}^{l_t}_{i=1}$, and that with true observed values, i.e., $\{(\bm{x}'_i,y'_i))\}^{l_t}_{i=1}$, is $\hat{\mu}_t(\bm{x})-\tilde{\mu}_t(\bm{x})=-\bm{p}(\bm{x})^{\mathrm{T}}\mathbf{M}(\hat {\bm{y}}'_{1:l_t}-\bm{y}'_{1:l_t})$. Furthermore, it holds that
\begin{equation}
\mathrm{Pr}\left( \forall 0 \leq t \leq T-1,\forall \bm{x}\in\mathcal{X}: |\hat{\mu}_t(\bm{x})-\tilde{\mu}_t(\bm{x})| \leq \Delta_m(L,l_t,\tau_t) \right)\ge1-dae^{-(L/b)^2}-\delta/4,
\end{equation}
where $\Delta_m(L,l_t,\tau_t)= l_t^2 \sqrt{1+\sigma^{-2}}\Big(Ld\tau_t/\sigma+ 2 \sqrt{\log \frac{4\sum\nolimits^{T-1}_{t=0} l_t}{\delta}}\Big)$.
\end{lemma}
\begin{proof}
	By Eq.~(\refeq{eq-mean}), we have
	\begin{align}
	\hat{\mu}_t(\bm{x})&=[\bm{k}_t(\bm{x});\bm{k}'_{l_t}(\bm{x})]^{\mathrm{T}}
	\left( \left[
	\begin{matrix}
	\mathbf{K}_t & \widetilde{\mathbf{K}}_{t,l_t}\\
	\widetilde{\mathbf{K}}_{t,l_t}^{\mathrm{T}} & \mathbf{K}'_{l_t}
	\end{matrix}
	\right]+\sigma^2\mathbf{I}\right)^{-1}[\bm{y}_{1:t};\hat{\bm{y}}_{1:l_t}']\\	&=\bm{k}_t(\bm{x})^{\mathrm{T}}(\mathbf{K}_{t}+\sigma^2\mathbf{I})^{-1}\bm{y}_{1:t}+\bm{p}(\bm{x})^{\mathrm{T}}\mathbf{M}([\mu_{t}(\bm{x}'_1);\mu_{t}(\bm{x}'_2);\ldots;\mu_{t}(\bm{x}'_{l_t})]-\hat{\bm{y}}'_{1:l_t})\\
	&=\mu_t(\bm{x})+\bm{p}(\bm{x})^{\mathrm{T}}\mathbf{M}([\mu_{t}(\bm{x}'_1);\mu_{t}(\bm{x}'_2);\ldots;\mu_{t}(\bm{x}'_{l_t})]-\hat{\bm{y}}'_{1:l_t}),
\end{align}
where the second equality is derived by the inverse of a block matrix, and the third one holds by Eq.~(\refeq{eq-mean}). Similarly, we have
\begin{align} \tilde{\mu}_t(\bm{x})&=\mu_t(\bm{x})+\bm{p}(\bm{x})^{\mathrm{T}}\mathbf{M}([\mu_{t}(\bm{x}'_1);\mu_{t}(\bm{x}'_2);\ldots;\mu_{t}(\bm{x}'_{l_t})]-\bm{y}'_{1:l_t}).
\end{align}
Thus, it holds that $\hat{\mu}_t(\bm{x})-\tilde{\mu}_t(\bm{x})=-\bm{p}(\bm{x})^{\mathrm{T}}\mathbf{M}(\hat{\bm{y}}'_{1:l_t}-\bm{y}'_{1:l_t})$.

Next, we examine the upper bound on $|\bm{p}(\bm{x})^{\mathrm{T}}\mathbf{M}(\hat{\bm{y}}'_{1:l_t}-\bm{y}'_{1:l_t})|$.
\begin{align}\label{eq-mid1}
&|\bm{p}(\bm{x})^{\mathrm{T}}\mathbf{M}(\hat{\bm{y}}'_{1:l_t}-\bm{y}'_{1:l_t})|\leq \sum_{i=1}^{l_t}| (\bm{p}(\bm{x})^{\mathrm{T}}\mathbf{M})_i |\cdot | \hat{y}_i'-y'_i |,
\end{align}
where $ (\bm{p}(\bm{x})^{\mathrm{T}}\mathbf{M})_i$ denotes the $i$-th element of $\bm{p}(\bm{x})^{\mathrm{T}}\mathbf{M}$. According to the procedure of Algorithm~\ref{alg:GBO}, the pseudo-point $\bm{x}'_i$ has distance $\tau_t$ with some observed data point and takes the same function value. Assume that the corresponding observed data point for $\bm{x}'_i$ is $\bm{x}_j$, where $j\in [t]$, implying $\hat{y}'_i=y_j$. Thus, we have
\begin{align}\label{eq-mid2}
&| \hat{y}_i'-y'_i | = | y_j-y'_i | = |(f(\bm{x}_j) + \epsilon_j) - (f(\bm{x}'_i) + \epsilon'_i)|
\leq |f(\bm{x}_j) - f(\bm{x}'_i)| + |\epsilon_j - \epsilon'_i|,
\end{align}
where $\epsilon_j,\epsilon'_i \sim \mathcal{N}(0,\sigma^2)$, and the second equality holds because $f(\bm{x}_j)$ and $f(\bm{x}'_i)$ are subject to additive Gaussian noise $\mathcal{N}(0,\sigma^2)$. According to Assumption~\ref{assump:lip}, we have
\begin{equation}\label{eq-mid3}
\mathrm{Pr}\left(\forall t\geq 0, \forall i \in [l_t]: |f(\bm{x}_j)-f(\bm{x}'_i)|\leq L\|\bm{x}_j-\bm{x}'_i\|_1= Ld\tau_t\right) \geq 1-dae^{-(L/b)^2}.
\end{equation}
As $\epsilon_j-\epsilon'_i \sim \mathcal{N}(0,2\sigma^2)$, by Lemma~\ref{lemma:ineq}, we have
\begin{equation}
\mathrm{Pr}\left(|\epsilon_j-\epsilon'_i|\leq 2\sigma\sqrt{\log \frac{4\sum\nolimits^{T-1}_{t=0} l_t}{\delta}}\right)\geq 1-\frac{\delta}{4\sum\nolimits^{T-1}_{t=0} l_t}.
\end{equation}
Applying the union bound leads to
	\begin{align}\label{eq-mid4}
	&\mathrm{Pr}\left(\forall 0 \leq t \leq T-1, \forall i \in [l_t]: | \hat{y_i}'-y'_i|\le Ld\tau_t+2\sigma\sqrt{\log \frac{4\sum\nolimits^{T-1}_{t=0} l_t}{\delta}} \right)\\
&\quad \ge1-dae^{-(L/b)^2}- \left(\sum^{T-1}_{t=0} l_t\right) \cdot \frac{\delta}{4\sum^{T-1}_{t=0} l_t} \geq 1-dae^{-(L/b)^2}-\frac{\delta}{4}.
	\end{align}
Thus, with probability at least $1-dae^{-(L/b)^2}-\delta/4$, it holds that $\forall 0 \leq t\leq T-1$,
\begin{align}\label{eq-mid5}
 |\bm{p}(\bm{x})^{\mathrm{T}}\mathbf{M}(\hat{\bm{y}}'_{1:l_t}-\bm{y}'_{1:l_t})|&\leq \left(Ld\tau_t+2\sigma\sqrt{\log \frac{4\sum\nolimits^{T-1}_{t=0} l_t}{\delta}}\right)\sum_{i=1}^{l_t}| (\bm{p}(\bm{x})^\mathrm{T}\mathbf{M})_i |.
\end{align}
Next we prove an upper bound on $ \sum_{i=1}^{l_t}|(\bm{p}(\bm{x})^\mathrm{T}\mathbf{M})_i|$. Note that
\begin{align}\label{eq-mid-2}
|(\bm{p}(\bm{x})^\mathrm{T}\mathbf{M})_i|\le \sum^{l_t}_{j=1}|\bm{p}(\bm{x})_j|\cdot |\mathbf{M}_{j,i}|,
\end{align}
where $\bm{p}(\bm{x})_j$ denotes the $j$-th element of $\bm{p}(\bm{x})$, and $\mathbf{M}_{j,i}$ denotes the element of the $j$-th row and $i$-th column of $\mathbf{M}$. If only one pseudo-point $(\bm{x}'_j, \hat{y}'_j)$ from $\{(\bm{x}'_i,\hat{y}'_i)\}^{l_t}_{i=1}$ is added into $D_t$, we know from Lemma~\ref{lemma:variance} that the reduction on the posterior variance is
	\begin{equation}
\Delta_v(\bm{x},1,\tau_t)=(\sigma_{t}^2(\bm{x}'_j)+\sigma^2)^{-1} \bm{p}(\bm{x})_j^2 \le 1,
	\end{equation}
where the inequality holds by $k(\bm{x},\bm{x})=1$. This implies
	\begin{equation}\label{eq-mid6}
	\forall j\in[l_t]: |\bm{p}(\bm{x})_j| \le \sqrt{\sigma_t^2(\bm{x}'_j)+\sigma^2}\le \sqrt{1+\sigma^2}.
	\end{equation}
Let $\mathrm{adj}(\cdot)$ denote the adjugate matrix, $[\cdot]_{i,j}$ denote the principle submatrix by deleting the $i$-th row and $j$-th column, and $\lambda_k(\cdot)$ denote the $k$-th largest eigenvalue. By Cramer's rule, $\mathbf{M}=\mathrm{adj}(\mathbf{M}^{-1})/\mathrm{det}(\mathbf{M}^{-1})$, and thus,
\begin{align}\label{eq-mid7}
|\mathbf{M}_{j,i}|=|\mathrm{adj}(\mathbf{M}^{-1})_{j,i}|/\mathrm{det}(\mathbf{M}^{-1})&=\mathrm{det}([\mathbf{M}^{-1}]_{i,j})/\mathrm{det}(\mathbf{M}^{-1})\\
&=\prod_{k=1}^{l_{t}-1}\lambda_k([\mathbf{M}^{-1}]_{i,j})/\prod_{k=1}^{l_t}\lambda_k(\mathbf{M}^{-1})\\
&\leq 1/\lambda_{l_t}(\mathbf{M}^{-1})\leq \sigma^{-2},
\end{align}
where the second equality holds by the definition of the adjugate matrix, and the two inequalities hold by the Cauchy interlacing inequality, leading to $\lambda_1(\mathbf{M}^{-1})\ge\lambda_1([\mathbf{M}^{-1}]_{i,j}) \ge\cdots\geq \lambda_{l_t-1}(\mathbf{M}^{-1})\ge\lambda_{l_t-1}([\mathbf{M}^{-1}]_{i,j})\ge\lambda_{l_t}(\mathbf{M}^{-1})\ge\sigma^{2}$. Combining Eqs.~(\refeq{eq-mid-2}),~(\refeq{eq-mid6}) and~(\refeq{eq-mid7}), we have
	\begin{align}\label{eq-mid-3}
	\sum_{i=1}^{l_t}| (\bm{p}(\bm{x})^\mathrm{T}\mathbf{M})_i |\leq\sum_{i=1}^{l_t} \sum_{j=1}^{l_t}|\bm{p}(\bm{x})_j|\cdot |\mathbf{M}_{j,i}|\le \sqrt{1+\sigma^2}\sigma^{-2}l_t^2.
	\end{align}
Applying Eq.~(\refeq{eq-mid-3}) to Eq.~(\refeq{eq-mid5}), we have
\begin{equation}
\begin{aligned}
&\mathrm{Pr}\Big( \forall 0 \leq t\leq T - 1,\forall \bm{x}\in\mathcal{X}: |\bm{p}(\bm{x})^{\mathrm{T}}\mathbf{M}(\hat{\bm{y}}'_{1:l_t} -\bm{y}'_{1:l_t})|  \\
&\qquad \leq \left(Ld\tau_t+ 2\sigma\sqrt{\log \frac{4\sum\nolimits^{T-1}_{t=0} l_t}{\delta}}\right)\sqrt{1+\sigma^2}\sigma^{-2}l_t^2\Big) \ge 1-dae^{-(L/b)^2}-\delta/4.
\end{aligned}
\end{equation}
Thus, the lemma holds.
\end{proof}

The proof of Theorem~\ref{theo:general} is inspired by that of Theorem~2 in~\citep{Srinivas2010GPUCB}, which gives an upper bound of UCB on the cumulative regret $R_T$. Their proof intuition is mainly that the instantaneous regret $r_t$ can be upper bounded by the width of confidence interval of $f(\bm{x}_t)$, relating to the posterior variance. The generation of pseudo-points will introduce another quantity into the upper bound on $r_t$, characterized by the error on the posterior mean in Lemma~\ref{lemma:mean}.

\begin{myproof}{Theorem~\ref{theo:general}}
According to Assumption~\ref{assump:lip} and $\beta_t=2\log(2\pi^2t^2(dt^2rL)^d/(3\delta))$, where $L=b\sqrt{\log(4da/\delta)}$, we can apply Lemma~5.7 in~\cite{Srinivas2010GPUCB} to derive that
\begin{align}\label{eq-mid-1}
&	\mathrm{Pr}(\forall t\ge 1: |f(\bm{x}^*)-\tilde{\mu}_{t-1}([\bm{x}^*]_t)|\le \beta_t^{1/2}\tilde{\sigma}_{t-1}([\bm{x}^*]_t)+1/t^2 )\ge 1-\delta/2,
\end{align}
where $[\bm{x}^*]_t$ denotes the discretized data point closest to $\bm{x}^*$ in the $t$-th iteration. Note that $\Delta_m(l_{t},\tau_{t})$ is just $\Delta_m(L,l_{t},\tau_{t})$ with $L=b\sqrt{\log(4da/\delta)}$ in Lemma~\ref{lemma:mean}. By the definition of $r_t$, we have
\begin{align}\label{eq-rt}
\forall t \geq 1: r_t&=f(\bm{x}^*)-f(\bm{x}_t)\\
&\le\beta_t^{1/2}\tilde{\sigma}_{t-1}([\bm{x}^*]_t)+\tilde{\mu}_{t-1}([\bm{x}^*]_t)-f(\bm{x}_t)+1/t^2\\
&\le  \beta_t^{1/2}\tilde{\sigma}_{t-1}([\bm{x}^*]_t)+\hat{\mu}_{t-1}([\bm{x}^*]_t)-f(\bm{x}_t)+1/t^2+\Delta_m(l_{t-1},\tau_{t-1})\\
&= \beta_t^{1/2}\hat{\sigma}_{t-1}([\bm{x}^*]_t)+\hat{\mu}_{t-1}([\bm{x}^*]_t)-f(\bm{x}_t)+1/t^2+\Delta_m(l_{t-1},\tau_{t-1})\\
&\leq \beta_t^{1/2}\hat{\sigma}_{t-1}(\bm{x}_t)+\hat{\mu}_{t-1}(\bm{x}_t)-f(\bm{x}_t)+1/t^2+\Delta_m(l_{t-1},\tau_{t-1})\\
&\leq \beta_t^{1/2}\hat{\sigma}_{t-1}(\bm{x}_t)+\tilde{\mu}_{t-1}(\bm{x}_t)-f(\bm{x}_t)+1/t^2+2\Delta_m(l_{t-1},\tau_{t-1})\\
&\leq 2\beta_t^{1/2}\hat{\sigma}_{t-1}(\bm{x}_t)+1/t^2+2\Delta_m(l_{t-1},\tau_{t-1}),
\end{align}
where the first inequality holds with probability at least $1-\delta/2$ by Eq.~(\refeq{eq-mid-1}), the second and fourth inequalities hold with probability at least $1-dae^{-(L/b)^2}-\delta/4=1-\delta/2$ by Lemma~\ref{lemma:mean}, the equality holds by $\forall \bm{x}: \hat{\sigma}_{t-1}(\bm{x})=\tilde{\sigma}_{t-1}(\bm{x})$, the third inequality holds because $\bm{x}_t$ is selected by maximizing $\hat{\mu}_{t-1}(\bm{x})+\beta_t^{1/2}\hat{\sigma}_{t-1}(\bm{x})$ in Eq.~(\refeq{eq-UCB}), and the last inequality holds with probability at least $1-\delta/4$ by Lemma~5.5 in~\cite{Srinivas2010GPUCB}. Note that to prove Lemma~5.7 in~\cite{Srinivas2010GPUCB} and Lemma~\ref{lemma:mean}, Assumption~\ref{assump:lip}, i.e., Eq.~(\refeq{eq-assumption}), is both used; thus, the probability $dae^{-(L/b)^2}=\delta/4$ has been repeated. By the union bound, we have
\begin{align}
\mathrm{Pr}(\forall t \geq 1: r_t \leq 2\beta_t^{1/2}\hat{\sigma}_{t-1}(\bm{x}_t)+1/t^2+2\Delta_m(l_{t-1},\tau_{t-1})) \geq 1-\delta/2-\delta/4-\delta/4=1-\delta,
\end{align}
implying
\begin{align}
&\qquad\; \mathrm{Pr}\Big(R_T=\sum\nolimits_{t=1}^{T}r_t \leq \sum\nolimits^T_{t=1}(2\beta_t^{1/2}\hat{\sigma}_{t-1}(\bm{x}_t)+1/t^2+2\Delta_m(l_{t-1},\tau_{t-1}))\Big) \geq 1-\delta.
\end{align}
By the Cauchy–Schwarz inequality, $C=8/\log(1+\sigma^{-2})$ and $\forall t \leq T: \beta_t \leq \beta_T$, we have
\begin{align}
\sum\nolimits_{t=1}^{T}2\beta_t^{1/2}\hat{\sigma}_{t-1}(\bm{x}_t)&\le \sqrt{T\sum\nolimits_{t=1}^{T}4\beta_t\hat{\sigma}^2_{t-1}(\bm{x}_t)}\\
&\le \sqrt{ \frac{CT\beta_T}{2}\sum\nolimits_{t=1}^{T}\log(1+\sigma^{-2}\hat{\sigma}^2_{t-1}(\bm{x}_t))}.
\end{align}
Let $PP_t$ denote the pseudo-points generated in the $t$-th iteration, and $\hat{\bm{y}}_{PP_t}$ denote their selected objective values. Let $H(\cdot)$ denote the entropy. We have
\begin{equation}
\begin{aligned}
&\frac{1}{2}\sum\limits_{t=1}^{T}\log(1+\sigma^{-2}\hat{\sigma}^2_{t-1}(\bm{x}_t)) + H(\bm{y}_{1:T}| \bm{f}_{1:T})\\
&=\frac{1}{2}\sum\limits_{t=1}^{T}\log(1+\sigma^{-2}\hat{\sigma}^2_{t-1}(\bm{x}_t)) + \frac{1}{2}\log(\det(2\pi e \sigma^2\mathbf{I}))\\
&=\frac{1}{2}\sum\limits_{t=1}^{T}\log(2\pi e(\sigma^2+\hat{\sigma}^2_{t-1}(\bm{x}_t)))\\
&=H(y_1\mid \hat{\bm{y}}_{PP_1}) + H(y_2\mid y_1,\hat{\bm{y}}_{PP_2}) + \cdots + H(y_T\mid \bm{y}_{1:T-1},\hat{\bm{y}}_{PP_T})\\
&=H(y_1\mid \hat{\bm{y}}_{PP}) + H(y_2\mid y_1,\hat{\bm{y}}_{PP}) + \cdots + H(y_T \mid \bm{y}_{1:T-1},\hat{\bm{y}}_{PP})\\
&=H(\bm{y}_{1:T}\mid \hat{\bm{y}}_{PP}),
\end{aligned}
\end{equation}
where the first equality holds because $f(\bm{x})$ is subject to additive Gaussian noise $\mathcal{N}(0,\sigma^2)$. Thus,
\begin{align}
\frac{1}{2}\sum\limits_{t=1}^{T}\log(1+\sigma^{-2}\hat{\sigma}^2_{t-1}(\bm{x}_t))&=H(\bm{y}_{1:T}\mid \hat{\bm{y}}_{PP}) - H(\bm{y}_{1:T}\mid \bm{f}_{1:T})\\
&=H(\hat{\bm{y}}_{PP}\mid \bm{y}_{1:T}) - H(\hat{\bm{y}}_{PP}) + H(\bm{y}_{1:T})-  H(\bm{y}_{1:T}\mid \bm{f}_{1:T})\\
&=I(\bm{y}_{1:T};\bm{f}_{1:T})-I(\bm{y}_{1:T};\hat{\bm{y}}_{PP})\\
&\leq \gamma'_T.
\end{align}
Considering $\sum_{t\geq 1} 1/t^2=\pi^2/6<2$, Eq.~(\refeq{eq-mid-5}) holds. Thus, the theorem holds.\vspace{0.5em}
\end{myproof}

Under the same assumption, it has been proved~\cite{Srinivas2010GPUCB} that the cumulative regret $R_T$ of UCB satisfies
\begin{align}\label{eq-mid-4}
\mathrm{Pr}\big(R_T \leq \sqrt{ CT\beta_T\gamma_T}+2\big) \geq 1-\delta,
\end{align}
where $\gamma_T=\max_{A: |A|=T} I(\bm{y}_A;\bm{f}_A)$, and the other parameters have the same meaning as that in Theorem~\ref{theo:general}. Our bound on $R_T$ of UCB-PP is actually a generalization of that of UCB. Without generating pseudo-points, $\forall t \geq 0: l_t=0 \wedge I(\bm{y}_{A};\hat{\bm{y}}_{PP})=0$, and thus, $\gamma'_T=\gamma_T \wedge \Delta_m(l_t,\tau_t)=0$, implying that Eq.~(\refeq{eq-mid-5}) specializes to Eq.~(\refeq{eq-mid-4}). As $\gamma'_T \leq \gamma_T$, the comparison between Eqs.~(\refeq{eq-mid-5}) and~(\refeq{eq-mid-4}) suggests that the generation of pseudo-points can be helpful if the negative influence of introducing the error on the posterior mean, i.e., introducing the term $\Delta_m(l_t,\tau_t)$, can be compensated by the positive influence of reducing the posterior variance, i.e., introducing the term $I(\bm{y}_{1:T};\hat{\bm{y}}_{PP})$.

\section{Empirical Study}\label{sec-exp}

In this section, we empirically compare BO-PP with BO. Three common acquisition functions, i.e., UCB, PI and EI, are used. The ARD squared exponential kernel is employed, whose hyper-parameters are tuned by maximum likelihood estimation (MLE), and the acquisition function is maximized via the DIRECT algorithm~\cite{Jones1993DIRECT}. To alleviate the ``cold start'' issue, each algorithm starts with five random initial points. To compare BO-PP with BO on each problem, we repeat their running 20 times independently and report the average results; in each running, BO-PP and BO use the same five random initial points. The noise level is set to $\sigma^2=0.0001$, and the iteration budget is set to 100.

In the $(t+1)$-th iteration of BO-PP, for each point in $D_{t}$, one pseudo-point is generated by randomly sampling within its distance $\tau_{t}$ and taking the same function value; thus, $l_t=|D_t|$. To control the error of objective values with pseudo-points increasing, $\tau_t$ is set to $r\tau_0/(dl_t)$, which decreases with $l_t$. Note that $r$ corresponds to the width of each dimension of the search domain. $\tau_0$ is set to a small value. We will use 0.01, 0.001 and 0.0001 to explore its influence, and the corresponding algorithms are denoted as BO-PP01, BO-PP001 and BO-PP0001, respectively.

We use four common synthetic benchmark functions: \textit{Dropwave}, \textit{Griewank}, \textit{Hart6} and \textit{Rastrigin}, whose dimensions are 2, 2, 6 and 2, respectively. Their search domains are scaled to $[-1,1]^d$. As the global minima are known, the simple regret $S_T$ is used as the metric. We also employ four real-world optimization problems, widely used in BO experiments~\cite{SpringenbergBONET16,Wang17MES,xu2018fast}. The first is to tune the hyper-parameters, i.e., box constraint $C\in[0.001,1000]$ and kernel scale $l\in[0.0001,1]$, of SVM for classification on the data set \textit{Wine quality} (1,599~\#inst, 11~\#feat). The second is to tune the hyper-parameters of 1-hidden-layer neural network (NN) for this task. The NN is trained by backpropagation, and the hyper-parameters are the number of neurons $n\in[1,100]$ and the learning rate $l_r\in [0.000001, 1]$. The last two problems are to tune the hyper-parameters of 1-hidden-layer NN for classification on \textit{Breast cancer} (699~\#inst, 9~\#feat) and regression on \textit{Boston housing} (506~\#inst, 13~\#feat), respectively. The NN is trained by Levenberg-Marquardt optimization, and there are four hyper-parameters: $n\in[1,100]$, the damping factor $\mu\in[0.000001, 100]$, the $\mu$-decrease and $\mu$-increase factors $\mu_{dec}\in [0.01,1]$, $\mu_{inc}\in[1.01, 20]$. All data sets are randomly split into training/validation/test sets with ratio 0.7/0.2/0.1, and the performance on validation sets is used as the objective $f$. For classification, $f$ is the classification accuracy; for regression, $f$ equals 20 minus the regression L2-loss. All codes and data sets can be downloaded from {\small\url{https://github.com/TELO19/BOPP}}.

For UCB, $\beta_t$ in Eq.~(\ref{eq-UCB}) is set to $2\log(t^{d/2+2}\pi^2/3\delta)$ where $\delta=0.1$, as suggested in~\cite{Brochu2010physical,Srinivas2010GPUCB}. For PI and EI, the best observed function value by far is used as $f(\bm{x}^+)$ in Eqs.~(\refeq{eq-PI}) and~(\refeq{eq-EI}). The results are summarized in Table~\ref{table_UCB_results}. We can observe that UCB-PP0001 is always better than UCB, verifying our theoretical analysis; UCB-PP01 and UCB-PP001 surpass UCB in most cases, disclosing that the performance of UCB-PP is not very sensitive to the distance $\tau_t$. Also, PI-PP and EI-PP perform better than PI and EI, respectively, in most cases, showing the applicability of generating pseudo-points.

\begin{table}[t]
	\centering
	\caption{The results (mean$\pm$std.) of BO-PP and BO on synthetic benchmark functions and real-world optimization problems, when reaching the iteration budget. $S_T$: the smaller, the better; $f$: the larger, the better. The bolded values denote that BO-PP is no worse than BO. UCB, PI and EI are tested.}\small
	\begin{tabular}{llllll}
		\toprule
		\multicolumn{2}{c}{Function}        &  UCB    &  UCB-PP01  &  UCB-PP001  &  UCB-PP0001	\\
		\hline
		\multirow{4}{0.01\linewidth}{$S_T$}&\textit{Dropwave}          & 0.2710$\pm$0.1311     &  \textbf{0.2232$\pm$0.1053}   & \textbf{0.1630$\pm$0.1014}   & \textbf{0.2121$\pm$0.1038}	\\
		&\textit{Griewank}      & 0.2357$\pm$0.2125          &   \textbf{0.2272$\pm$0.1644}    &  \textbf{0.2350$\pm$0.1690}    & \textbf{0.2085$\pm$0.1177	}\\
&\textit{Hart6}  & 1.0256$\pm$0.3498          &   1.0565$\pm$0.3620      &   1.0868$\pm$0.3153       & \textbf{0.9276$\pm$0.3307 }         	\\
		&\textit{Rastrigin}		  &    3.3492$\pm$3.2602      &   3.6975$\pm$2.7991        &      3.5124$\pm$2.4124      &  \textbf{3.0077$\pm$2.3245}			\\
		\hline
		\multirow{4}{0.01\linewidth}{$f$}& \textit{SVM$\_$wine}& 0.6182$\pm$0.0029  & \textbf{0.6186$\pm$0.0030} & \textbf{0.6186$\pm$0.0036} & \textbf{0.6189$\pm$0.0042} \\
		&\textit{NN$\_$wine}		  &    0.9149$\pm$0.0004    &   \textbf{0.9151$\pm$0.0005}       &       \textbf{0.9151$\pm$0.0004}    &     \textbf{0.9151$\pm$0.0004}        \\
		&\textit{NN$\_$cancer}	  &	0.9585$\pm$0.0006	    &  \textbf{0.9589$\pm$0.0006}  &  \textbf{0.9589$\pm$0.0006}   &\textbf{0.9590$\pm$0.0006}			\\
		&\textit{NN$\_$housing}		  &    8.6733$\pm$1.6916     &   \textbf{8.6776$\pm$1.7149}        &      \textbf{9.0691$\pm$1.8656}      &  \textbf{8.7216$\pm$1.8076}			\\
\hline\hline
\multicolumn{2}{c}{Function}          &  PI    &  PI-PP01  &  PI-PP001  &  PI-PP0001	\\
		\hline
			\multirow{4}{0.01\linewidth}{$S_T$}& \textit{Dropwave}          & 0.1526$\pm$0.1534          & \textbf{0.1221$\pm$0.1462}    & \textbf{0.1251$\pm$0.1355}      & \textbf{0.1457$\pm$0.1539}			\\
		&\textit{Griewank}      & 0$\pm$0           & \textbf{0$\pm$0}          & \textbf{0$\pm$0}         & \textbf{0$\pm$0}			\\		
		&\textit{Hart6}  & 0.5795$\pm$0.2959     & \textbf{0.4558$\pm$0.1048}   & \textbf{0.5599$\pm$0.0982}  & \textbf{0.5500$\pm$0.2529}    \\		
		&\textit{Rastrigin}		  &    0.0524$\pm$0.2285      &      \textbf{0.0524$\pm$0.2285}     &      \textbf{0$\pm$0}      &  	\textbf{0.0524$\pm$0.2285}		\\
		\hline
		\multirow{4}{0.01\linewidth}{$f$}&\textit{SVM$\_$wine}& 0.6192$\pm$0.0037  & 0.6176$\pm$0.0025 & \textbf{0.6204$\pm$0.0050}& \textbf{0.6208$\pm$0.0053}\\
		&\textit{NN$\_$wine}		  &    0.9140$\pm$0.0011    &    \textbf{0.9143$\pm$0.0006}      &        \textbf{0.9140$\pm$0.0008}   &     0.9138$\pm$0.0008        \\
		&\textit{NN$\_$cancer}	  &	0.9571$\pm$0.0024   &  \textbf{0.9578$\pm$0.0020}  &   \textbf{0.9576$\pm$0.0024}  & \textbf{0.9574$\pm$0.0024}		\\
		&\textit{NN$\_$housing}	  &7.5570$\pm$1.4822   & \textbf{7.9702$\pm$1.4000}  &  \textbf{7.8585$\pm$1.5515}    &\textbf{7.6147$\pm$1.3592}	\\
\hline\hline
\multicolumn{2}{c}{Function}          &  EI  &    EI-PP01  &  EI-PP001  &  EI-PP0001	\\
		\hline
		\multirow{4}{0.01\linewidth}{$S_T$}&\textit{Dropwave}          & 0.2557$\pm$0.1720      &  \textbf{0.1924$\pm$0.0818 }    & \textbf{0.2307$\pm$0.1461}   & \textbf{0.2276$\pm$0.1752}	\\
		&\textit{Griewank}      & 0.3098$\pm$0.1722        &  \textbf{0.3028$\pm$0.1005}    & 0.3187$\pm$0.1594     &\textbf{0.2729$\pm$0.1471}		\\
		&\textit{Hart6} & 0.6652$\pm$0.2685        &  \textbf{0.6050$\pm$0.2328}    &  \textbf{0.6028$\pm$0.1656}   &   0.6828$\pm$0.3081        	\\	
		&\textit{Rastrigin}		  &    3.3069$\pm$2.4955      & \textbf{2.6602$\pm$2.2063}  & \textbf{3.0492$\pm$1.5602}   &  \textbf{3.1987$\pm$2.3818}	\\
		\hline
		\multirow{4}{0.01\linewidth}{$f$}&\textit{SVM$\_$wine}& 0.6189$\pm$0.0037  & 0.6182$\pm$0.0035 & \textbf{0.6198$\pm$0.0037} &  0.6179$\pm$0.0034\\
		&\textit{NN$\_$wine}		  &    0.9149$\pm$0.0006     &   \textbf{0.9150$\pm$0.0005}       &       \textbf{0.9150$\pm$0.0004}   &     0.9148$\pm$0.0005        \\
&\textit{NN$\_$cancer}	  &	0.9587$\pm$0.0006	    &  \textbf{0.9589$\pm$0.0006} &  \textbf{0.9588$\pm$0.0006}    & \textbf{0.9587$\pm$0.0006}	\\
		&\textit{NN$\_$housing}	  &    8.1780$\pm$1.8382   &    8.0277$\pm$1.3844       &   8.0471$\pm$1.5024         & \textbf{8.1992$\pm$1.7971}	\\
		\bottomrule
	\end{tabular}\label{table_UCB_results}
\end{table}

Furthermore, we plot the curves of the simple regret $S_T$ or the objective $f$ over iterations for each algorithm on each problem, as shown in Figures~\ref{fig-UCB-real-world} to~\ref{fig-EI-benchmark}. Figure~\ref{fig-UCB-real-world} shows the curves of UCB-PP and UCB on real-world problems. It can be observed that on each problem, there is at least one curve of UCB-PP almost always above that of UCB, implying that UCB-PP can consistently outperform UCB during the running process. The other five figures show similar observations.

\begin{figure*}[t]\centering
\begin{minipage}[c]{0.5\linewidth}\centering
	\includegraphics[width=1\linewidth]{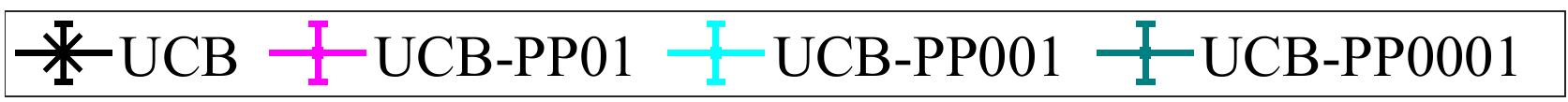}
\end{minipage}
\\\vspace{0.2em}
\begin{minipage}[c]{0.24\linewidth}\centering
	\includegraphics[width=1\linewidth]{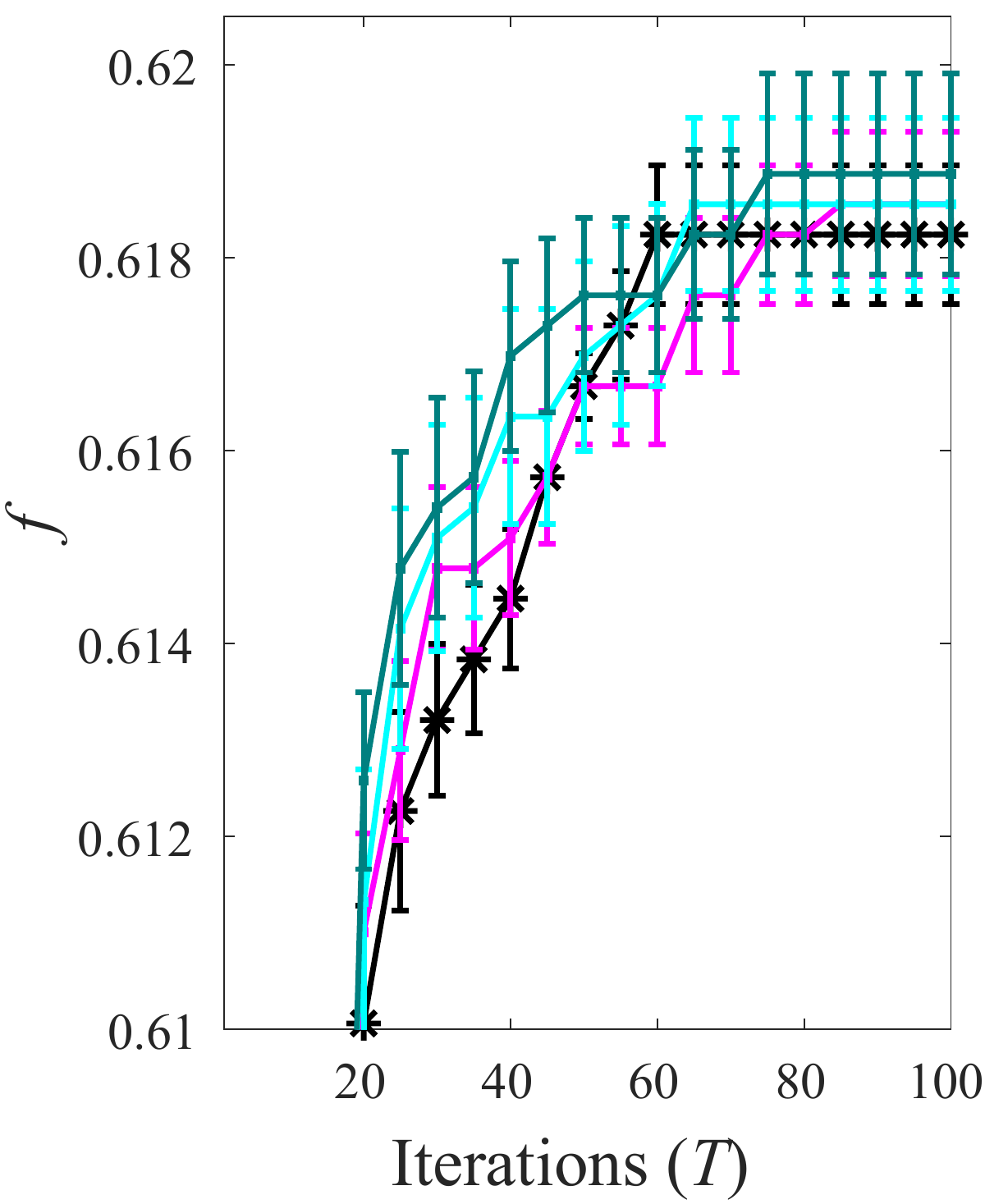}
\end{minipage}\ \
\begin{minipage}[c]{0.24\linewidth}\centering
	\includegraphics[width=1\linewidth]{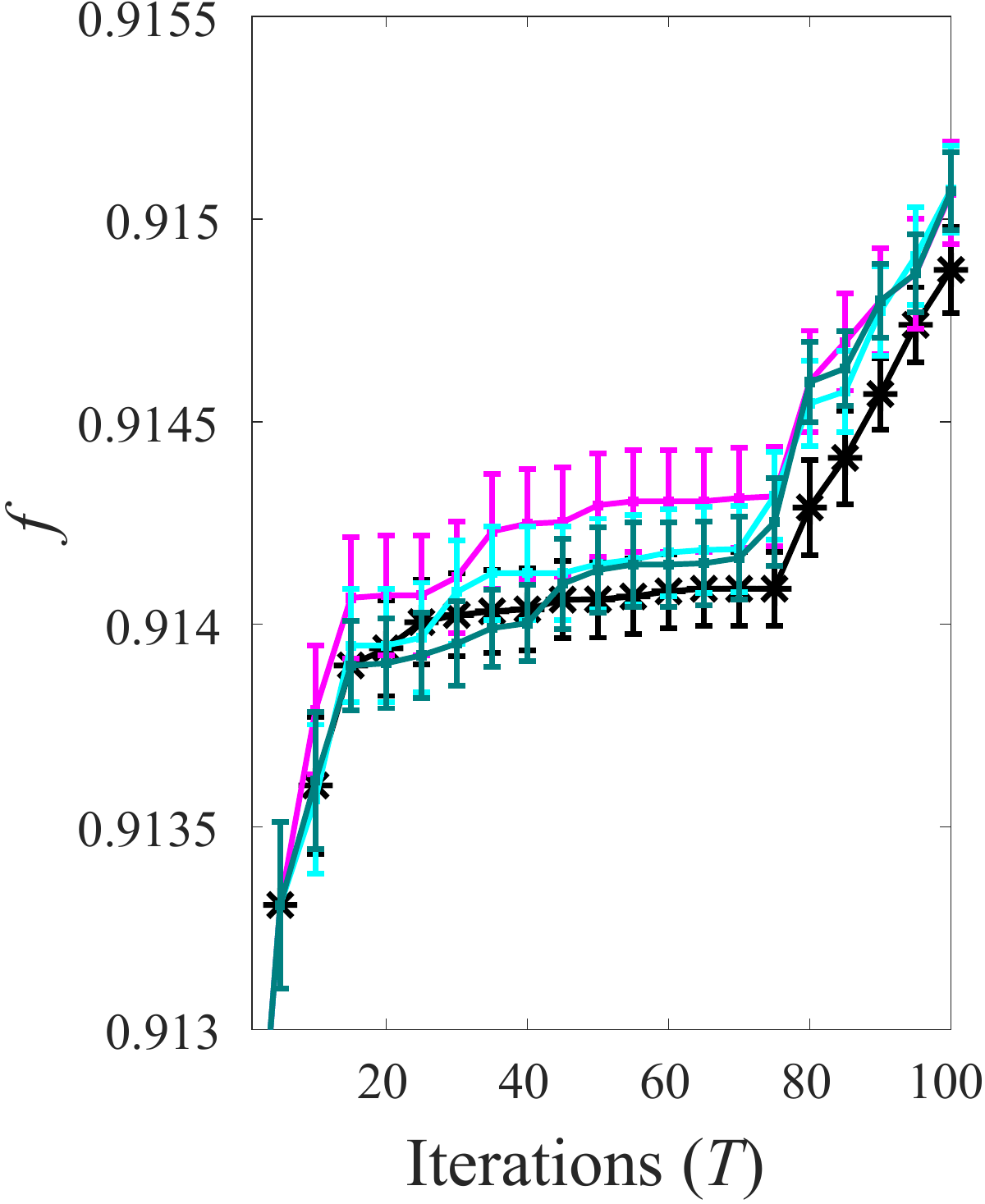}
\end{minipage}\ \
\begin{minipage}[c]{0.24\linewidth}\centering
	\includegraphics[width=1\linewidth]{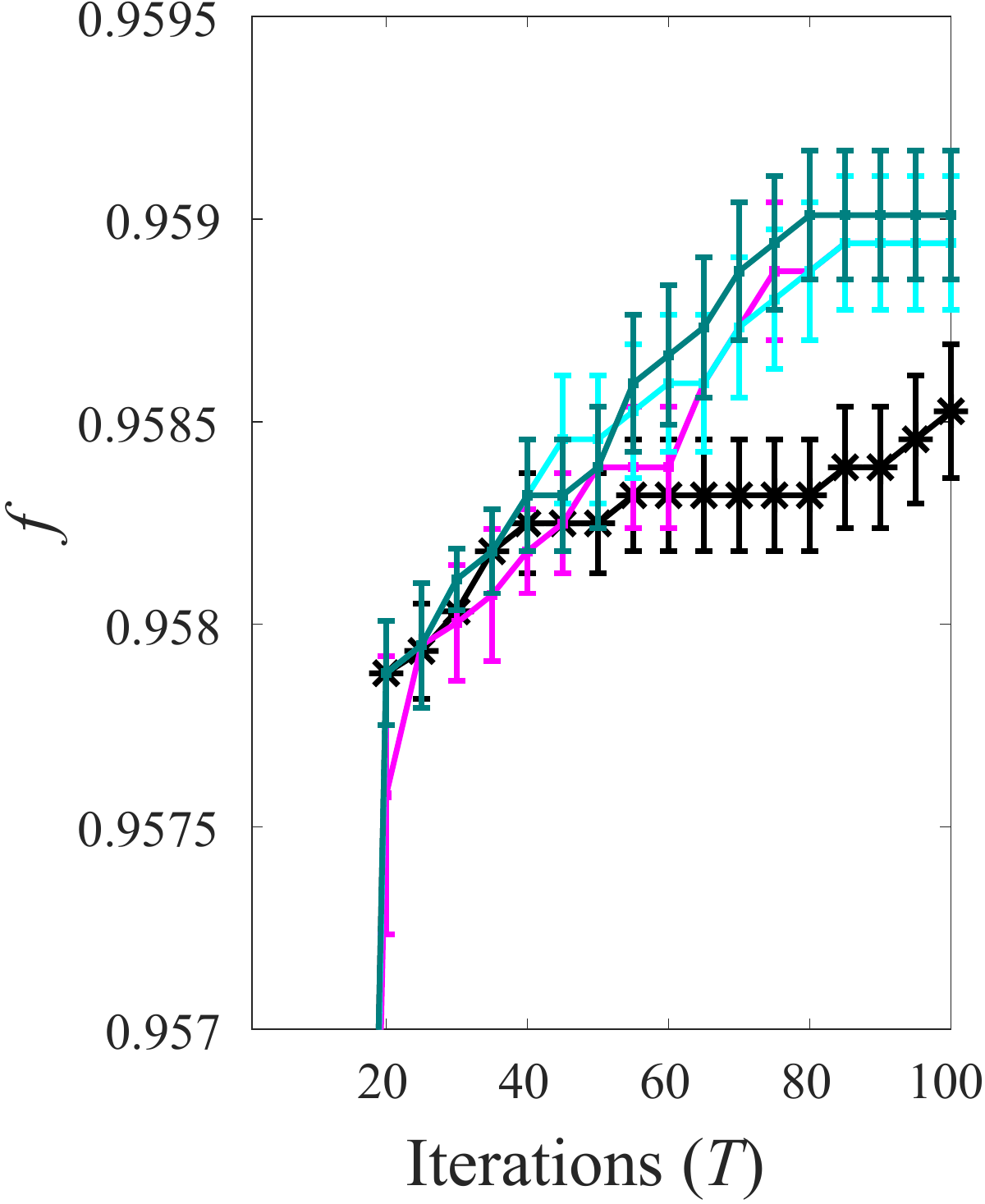}
\end{minipage}
\begin{minipage}[c]{0.24\linewidth}\centering
	\includegraphics[width=1\linewidth]{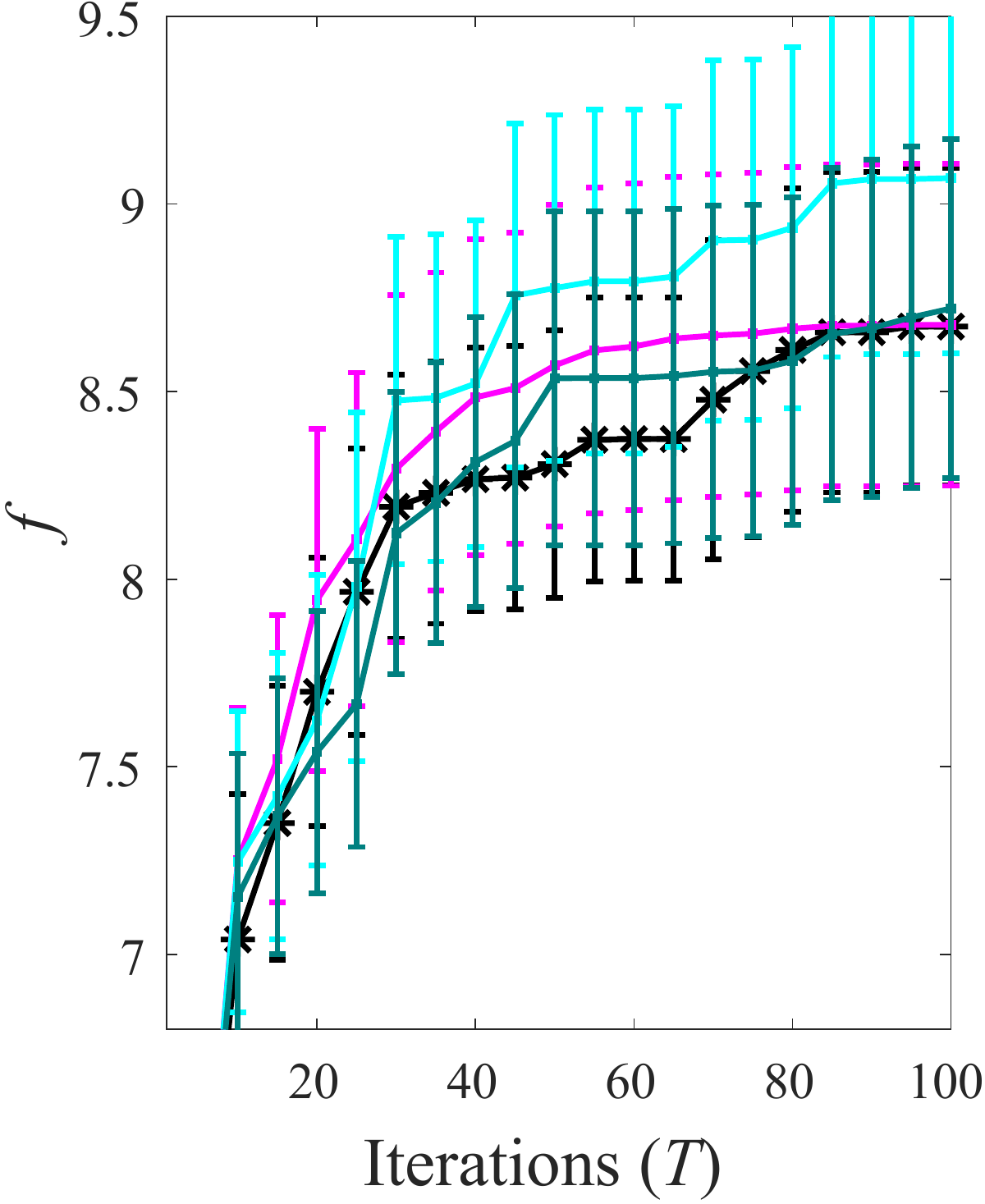}
\end{minipage}\ \
\\\vspace{0.2em}
\begin{minipage}[c]{0.24\linewidth}\centering
	\small(a) \textit{SVM$\_$wine}
\end{minipage}\ \
\begin{minipage}[c]{0.24\linewidth}\centering
	\small(b) \textit{NN$\_$wine}
\end{minipage}\ \
\begin{minipage}[c]{0.24\linewidth}\centering
	\small(c) \textit{NN$\_$cancer}
\end{minipage}
\begin{minipage}[c]{0.24\linewidth}\centering
	\small(d) \textit{NN$\_$housing}
\end{minipage}
\caption{The results (mean$\pm$(1/4)std.) of UCB-PP and UCB on real-world optimization problems. $f$: the larger, the better.}\label{fig-UCB-real-world}
\end{figure*}

\begin{figure*}[t!]\centering
	\begin{minipage}[c]{0.5\linewidth}\centering
		\includegraphics[width=1\linewidth]{UCB_legend}
	\end{minipage}
	\\\vspace{0.2em}
	\begin{minipage}[c]{0.24\linewidth}\centering
		\includegraphics[width=1\linewidth]{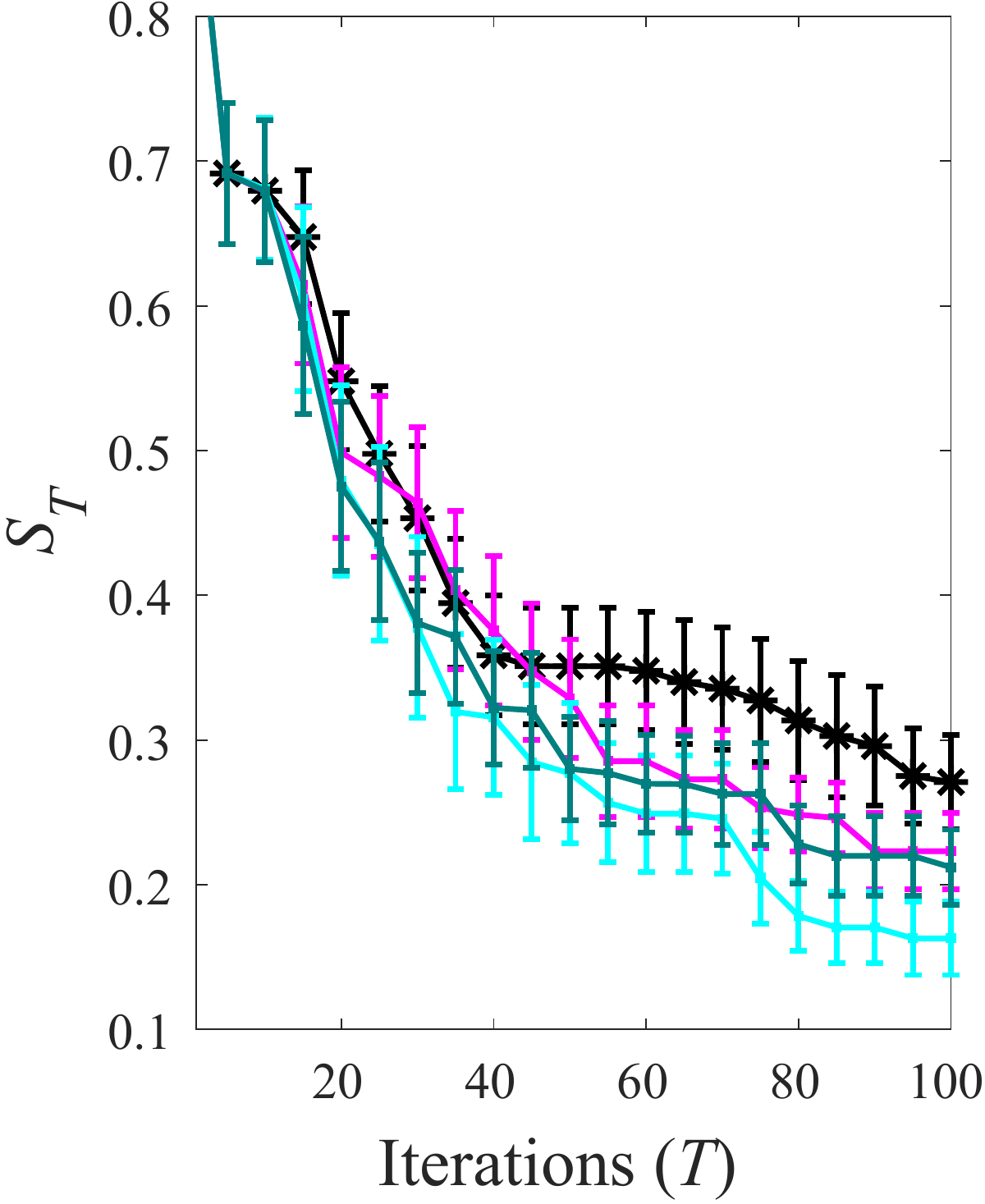}
	\end{minipage}\ \
	\begin{minipage}[c]{0.24\linewidth}\centering
		\includegraphics[width=1\linewidth]{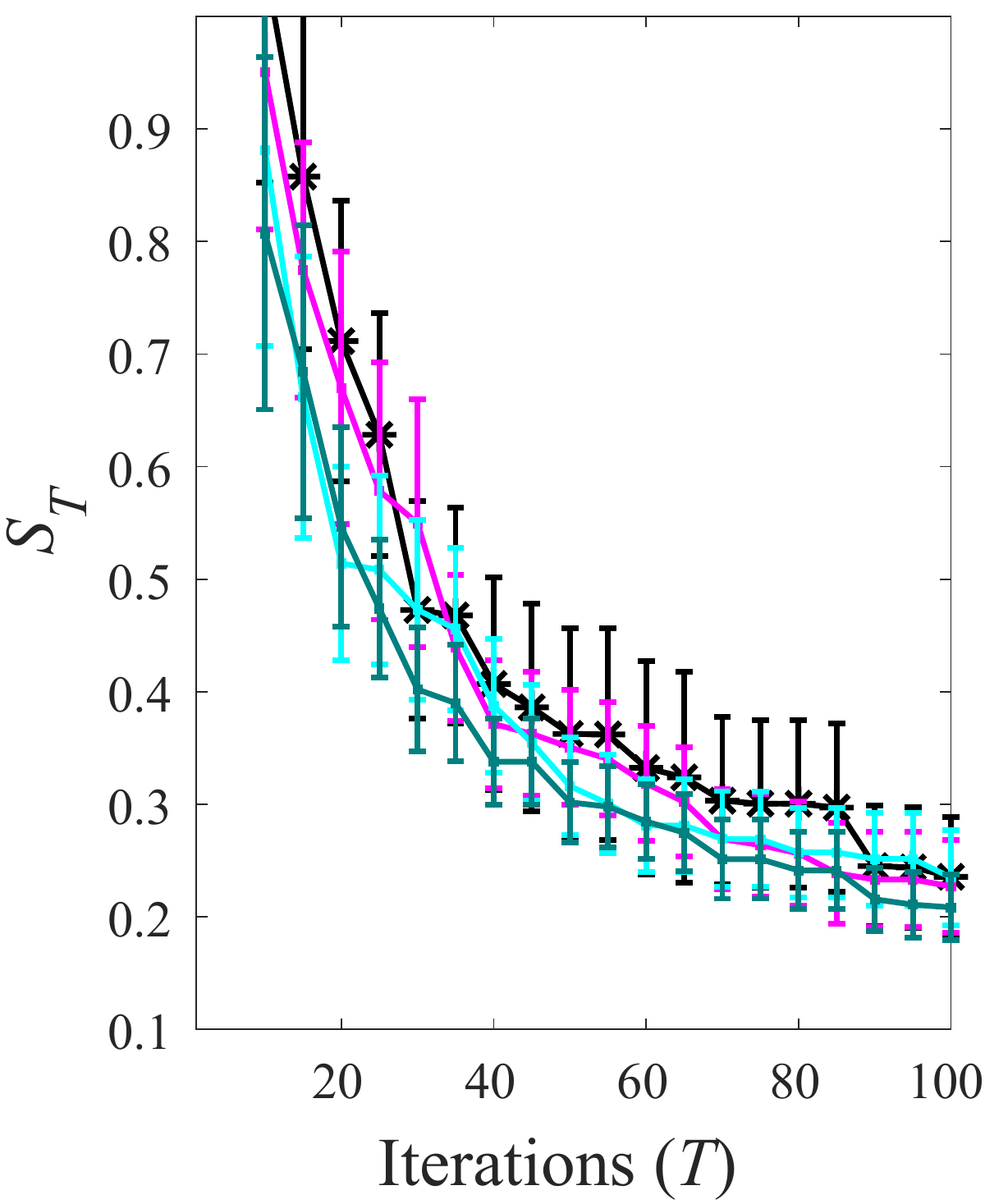}
	\end{minipage}\ \
	\begin{minipage}[c]{0.24\linewidth}\centering
		\includegraphics[width=1\linewidth]{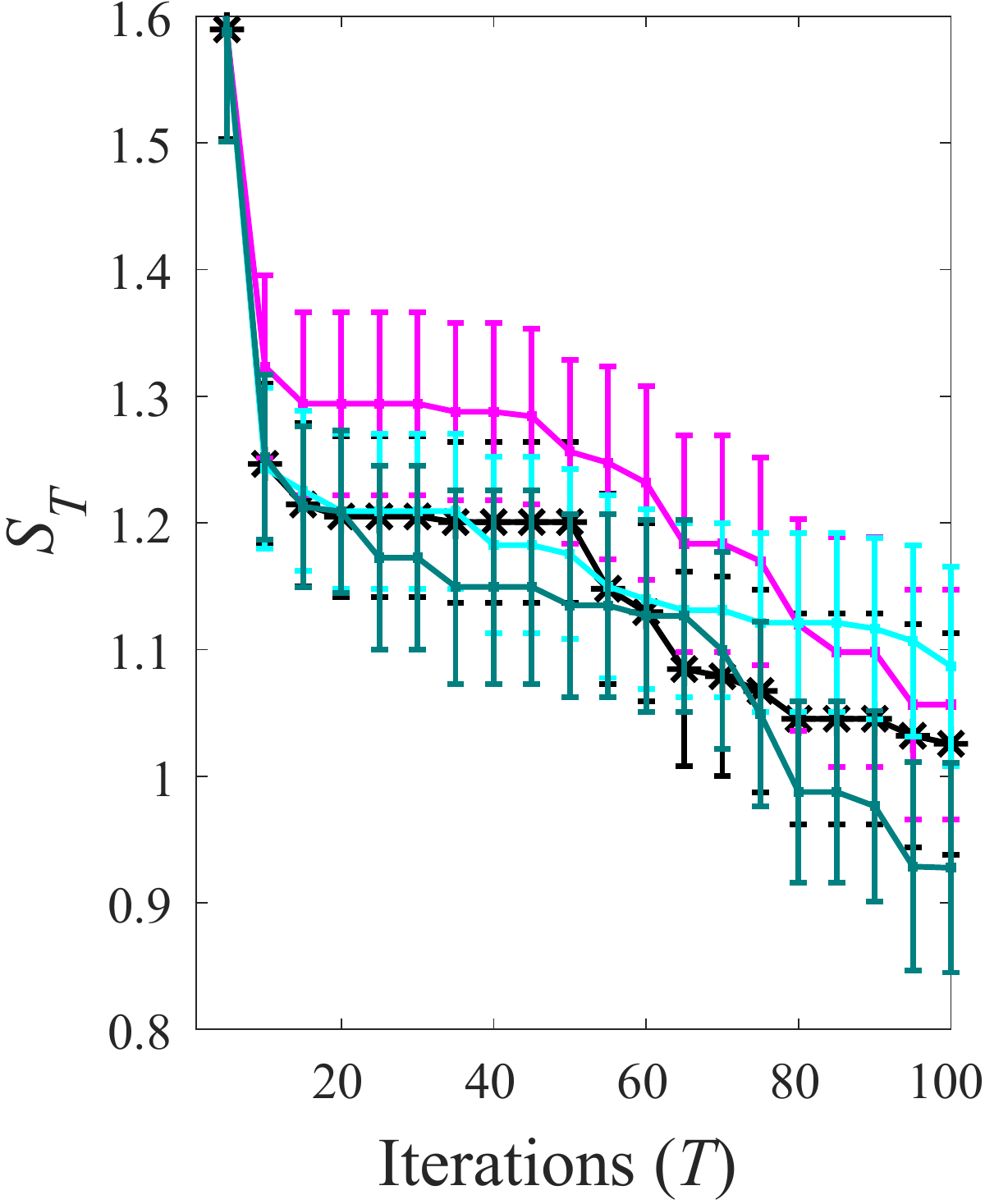}
	\end{minipage}
	\begin{minipage}[c]{0.24\linewidth}\centering
		\includegraphics[width=1\linewidth]{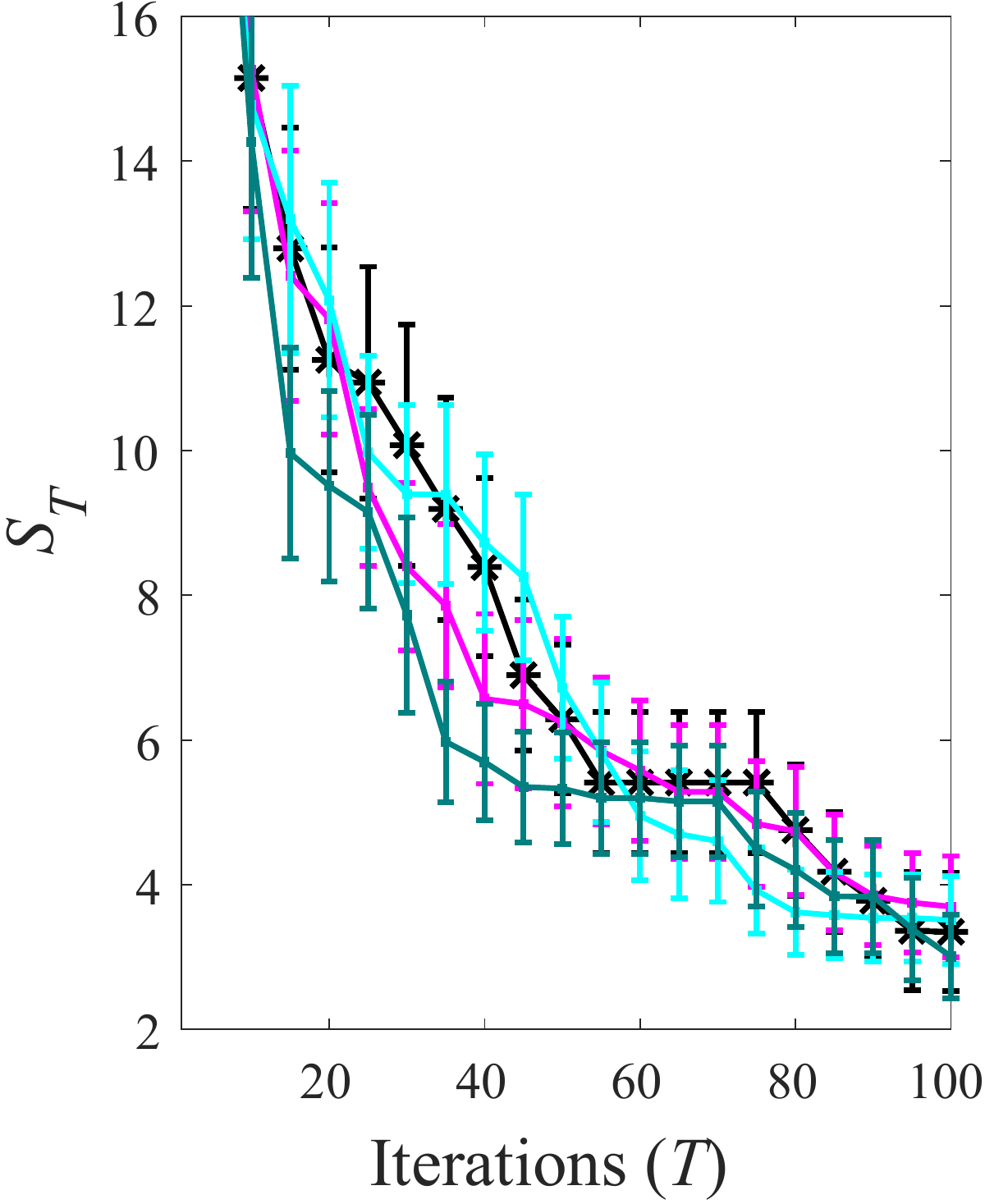}
	\end{minipage}
	\\\vspace{0.2em}
	\begin{minipage}[c]{0.24\linewidth}\centering
		\small(a) \textit{Dropwave}
	\end{minipage}\ \
	\begin{minipage}[c]{0.24\linewidth}\centering
		\small(b) \textit{Griewank}
	\end{minipage}\ \
	\begin{minipage}[c]{0.24\linewidth}\centering
		\small(c) \textit{Hart6}
	\end{minipage}
	\begin{minipage}[c]{0.24\linewidth}\centering
		\small(d) \textit{Rastrigin}
	\end{minipage}
	\caption{The results (mean$\pm$(1/4)std.) of UCB-PP and UCB on synthetic benchmark functions. $S_T$: the smaller, the better.}\label{fig-UCB-benchmark}
\end{figure*}

\begin{figure*}[t!]\centering
\begin{minipage}[c]{0.45\linewidth}\centering
	\includegraphics[width=1\linewidth]{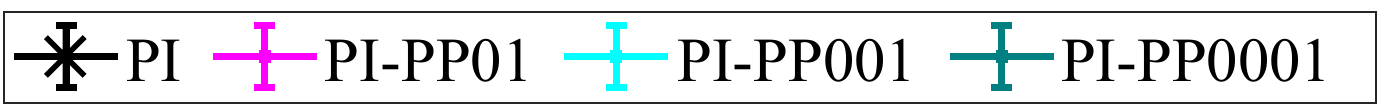}
\end{minipage}
\\\vspace{0.2em}
\begin{minipage}[c]{0.24\linewidth}\centering
	\includegraphics[width=1\linewidth]{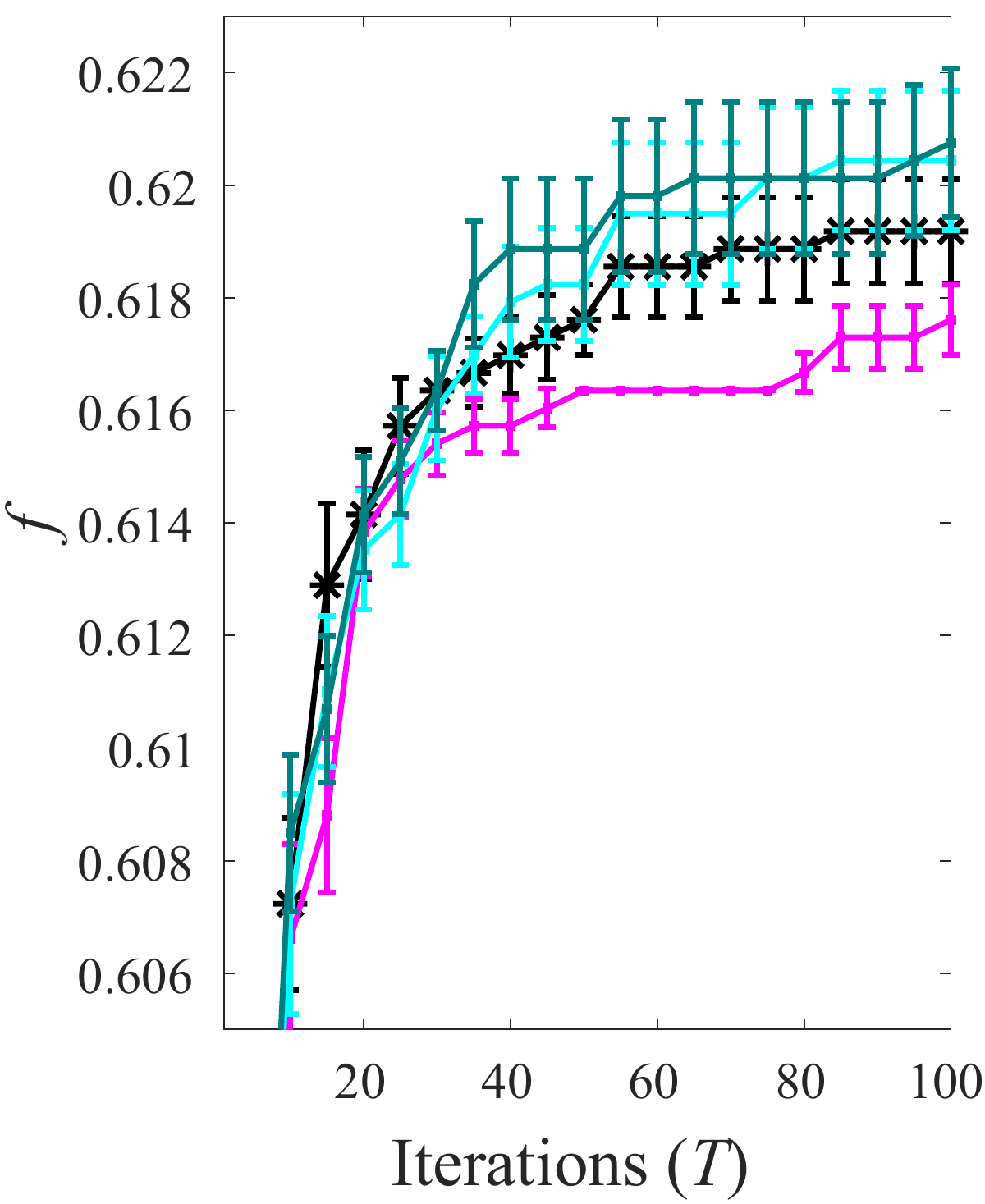}
\end{minipage}\ \
\begin{minipage}[c]{0.24\linewidth}\centering
	\includegraphics[width=1\linewidth]{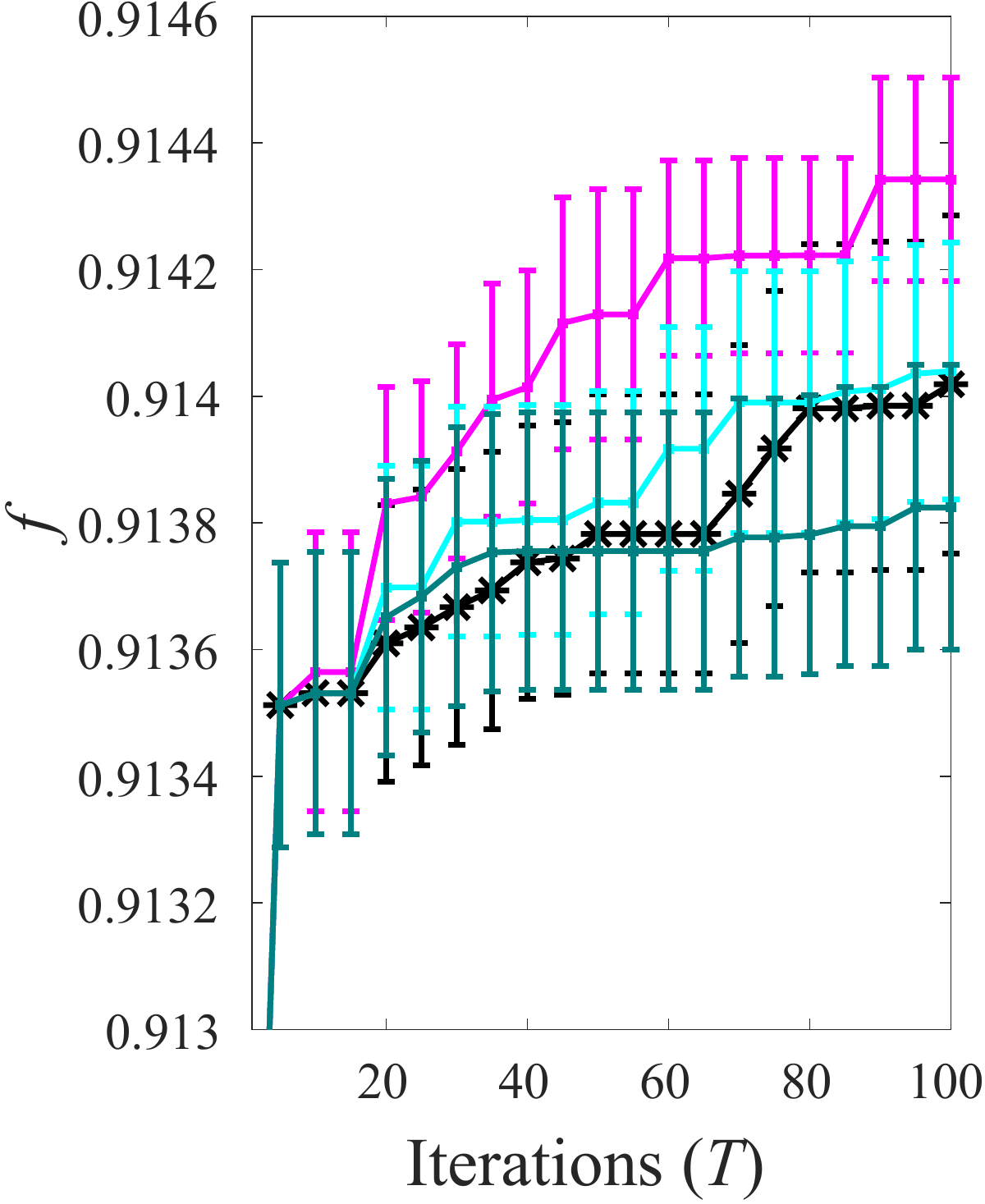}
\end{minipage}\ \
\begin{minipage}[c]{0.24\linewidth}\centering
	\includegraphics[width=1\linewidth]{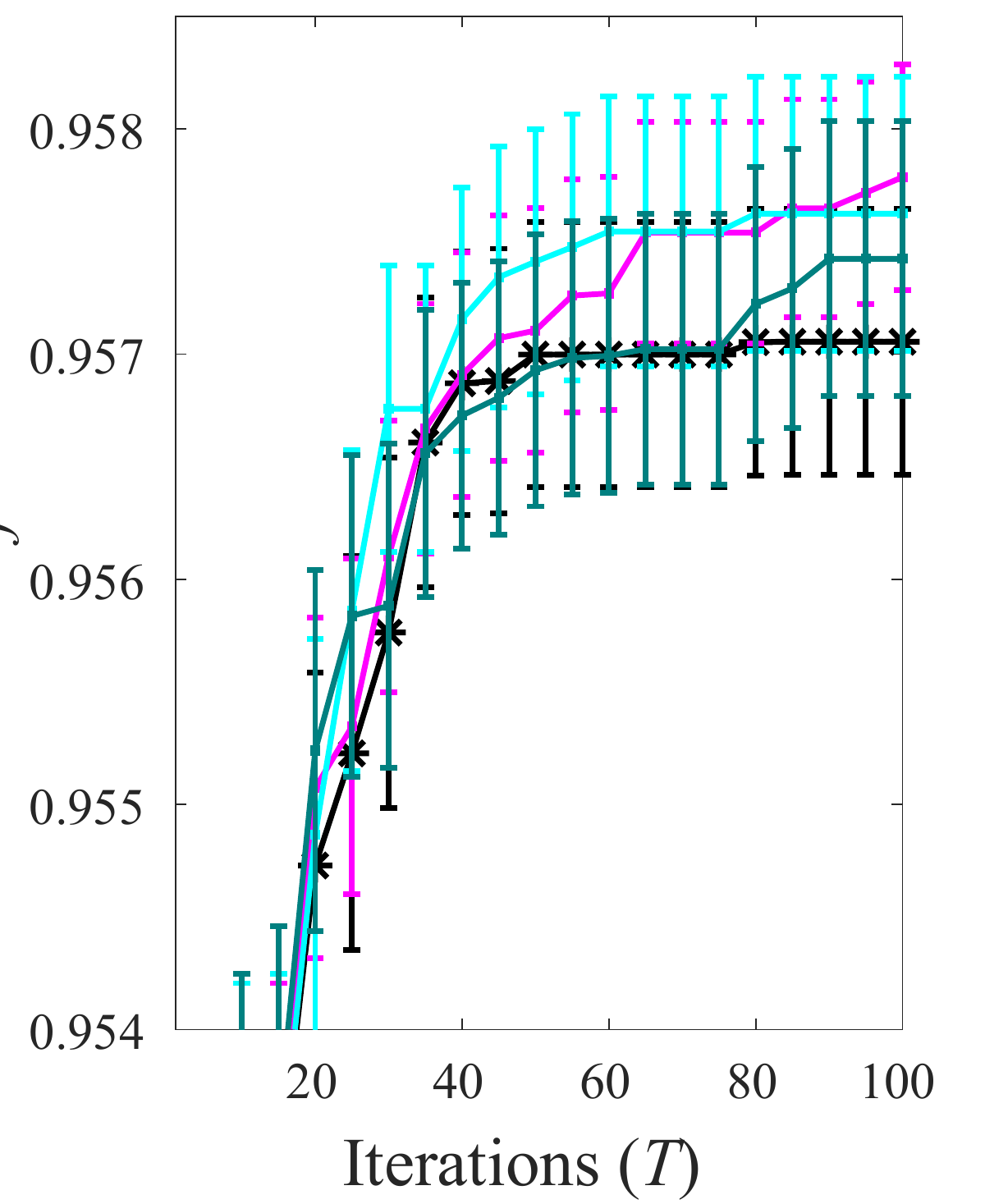}
\end{minipage}
\begin{minipage}[c]{0.24\linewidth}\centering
	\includegraphics[width=1\linewidth]{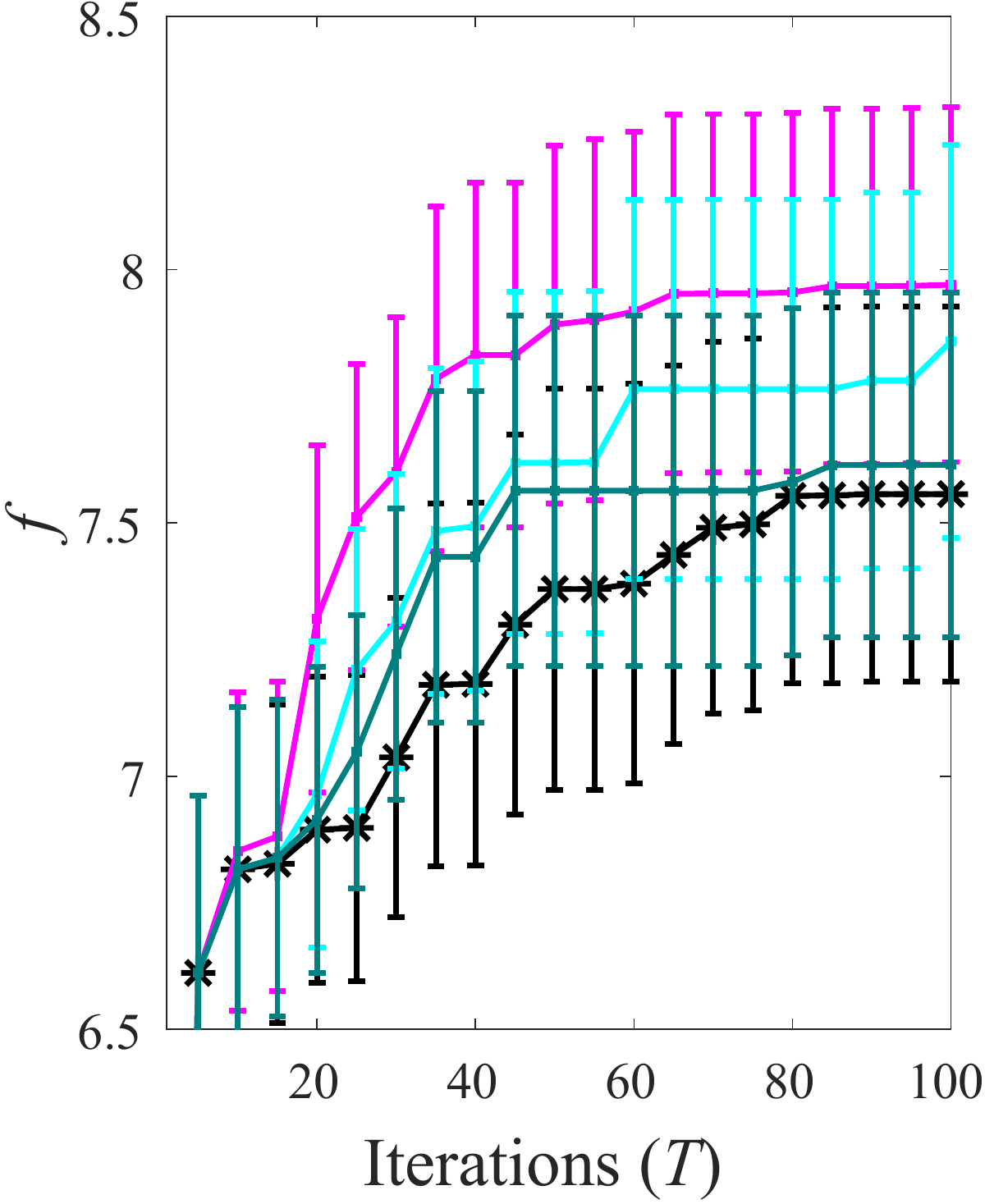}
\end{minipage}\ \
\\\vspace{0.2em}
\begin{minipage}[c]{0.24\linewidth}\centering
	\small(a) \textit{SVM$\_$wine}
\end{minipage}\ \
\begin{minipage}[c]{0.24\linewidth}\centering
	\small(b) \textit{NN$\_$wine}
\end{minipage}\ \
\begin{minipage}[c]{0.24\linewidth}\centering
	\small(d) \textit{NN$\_$cancer}
\end{minipage}
\begin{minipage}[c]{0.24\linewidth}\centering
	\small(c) \textit{NN$\_$housing}
\end{minipage}\ \
	\caption{The results (mean$\pm$(1/4)std.) of PI-PP and PI on real-world optimization problems. $f$: the larger, the better.}\label{fig-PI-real-world}
\end{figure*}

\begin{figure*}[t!]\centering
	\begin{minipage}[c]{0.45\linewidth}\centering
		\includegraphics[width=1\linewidth]{PI_legend}
	\end{minipage}
	\\\vspace{0.2em}
	\begin{minipage}[c]{0.24\linewidth}\centering
		\includegraphics[width=1\linewidth]{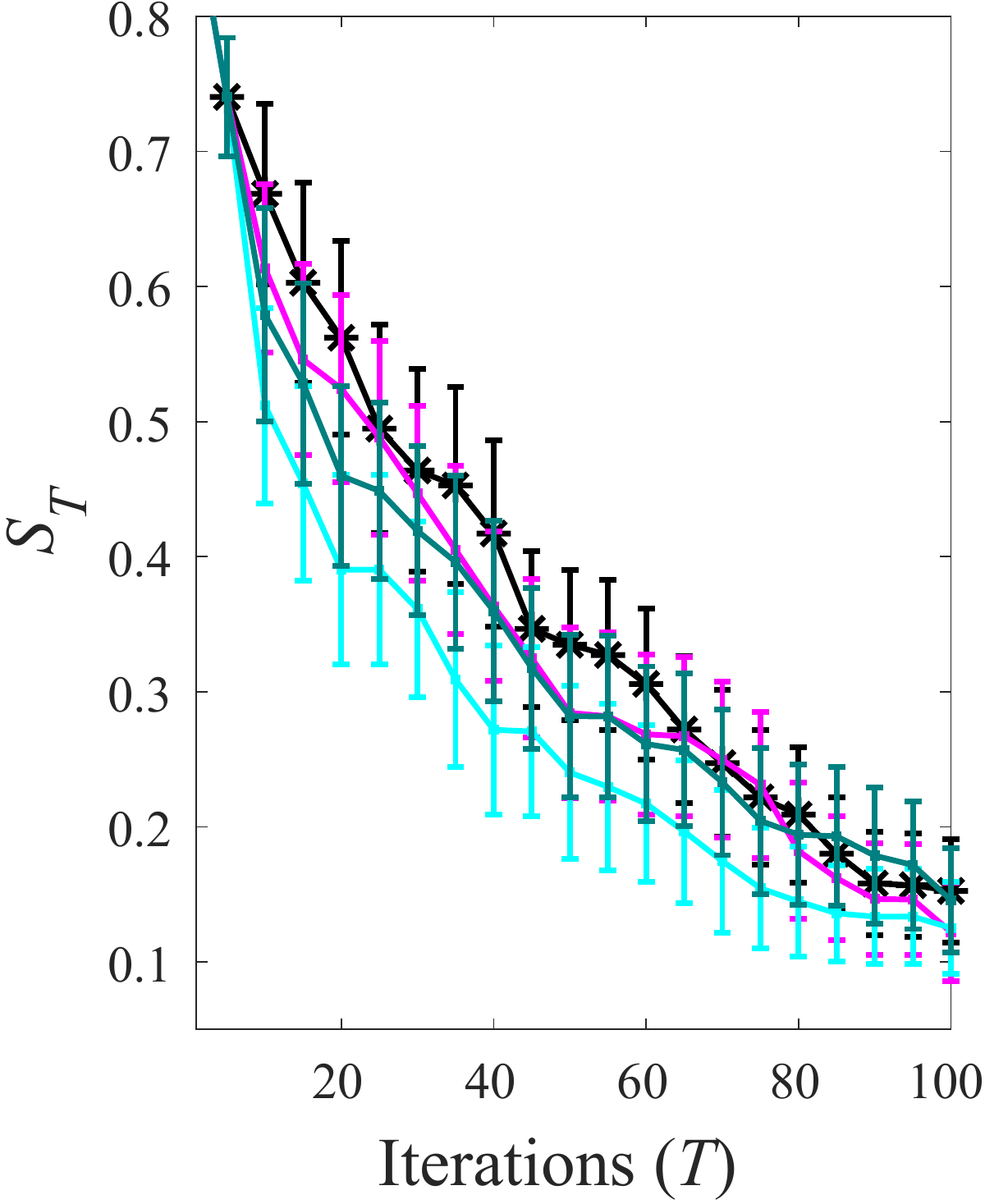}
	\end{minipage}\ \
	\begin{minipage}[c]{0.24\linewidth}\centering
		\includegraphics[width=1\linewidth]{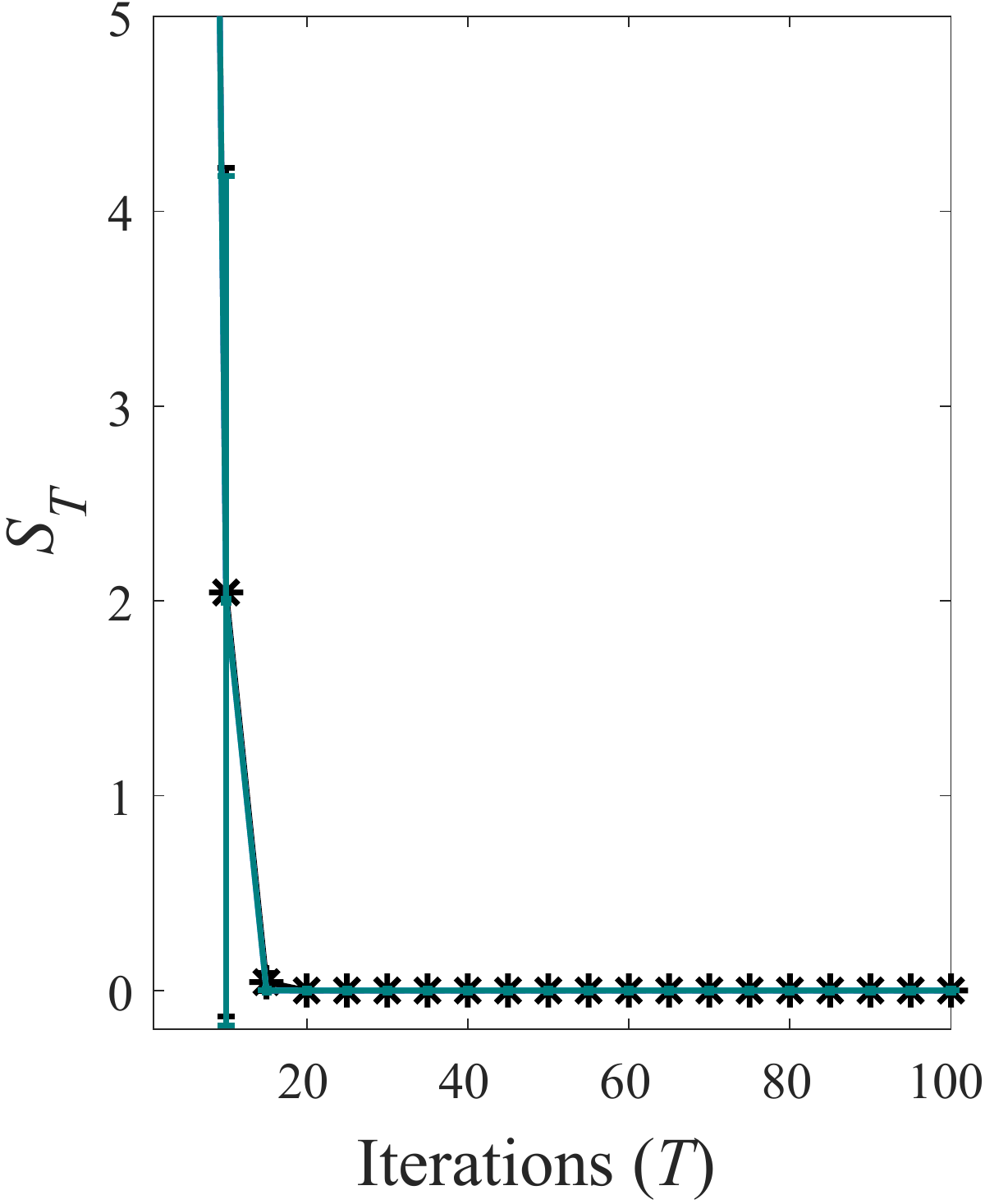}
	\end{minipage}\ \
	\begin{minipage}[c]{0.24\linewidth}\centering
		\includegraphics[width=1\linewidth]{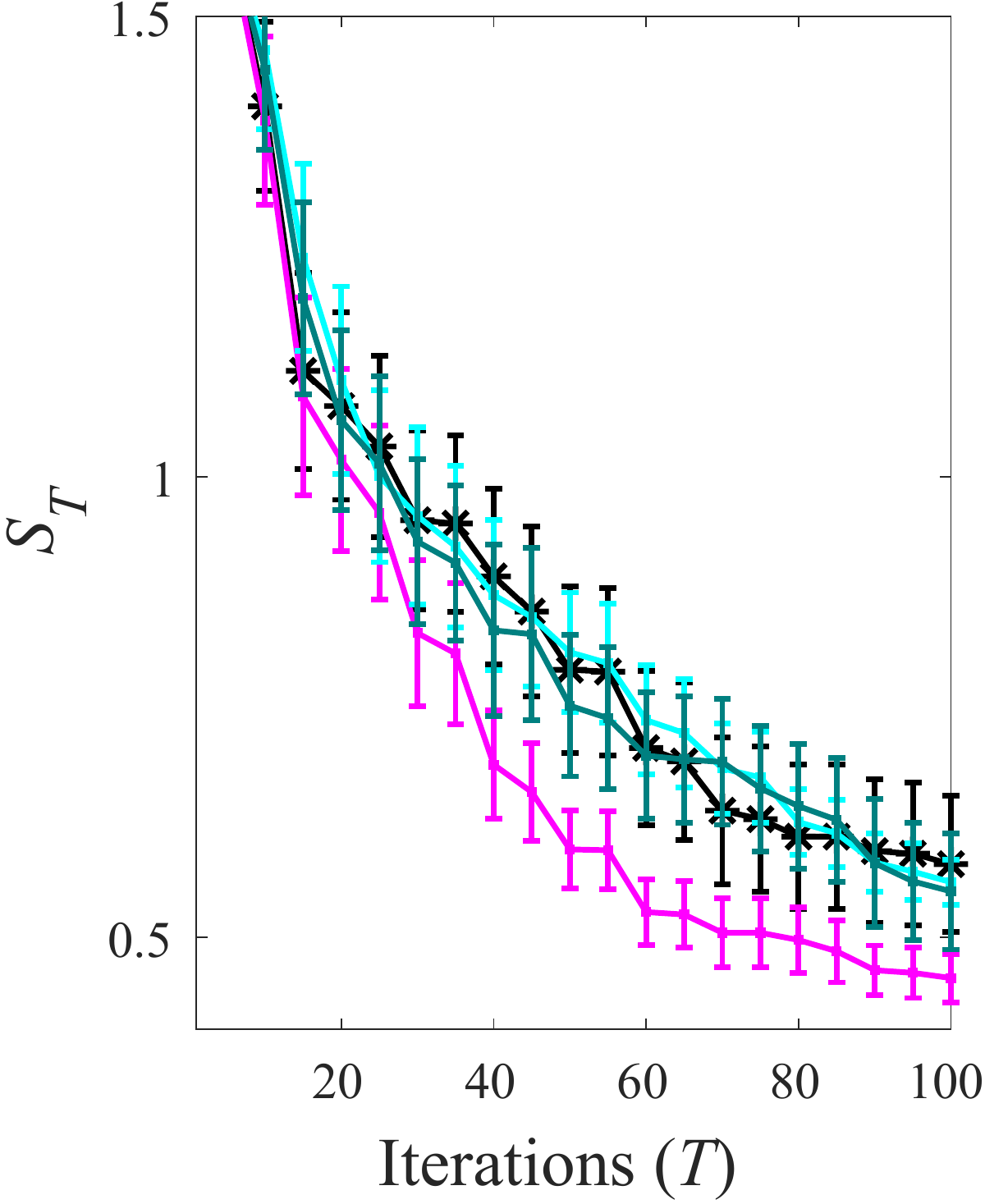}
	\end{minipage}
	\begin{minipage}[c]{0.24\linewidth}\centering
		\includegraphics[width=1\linewidth]{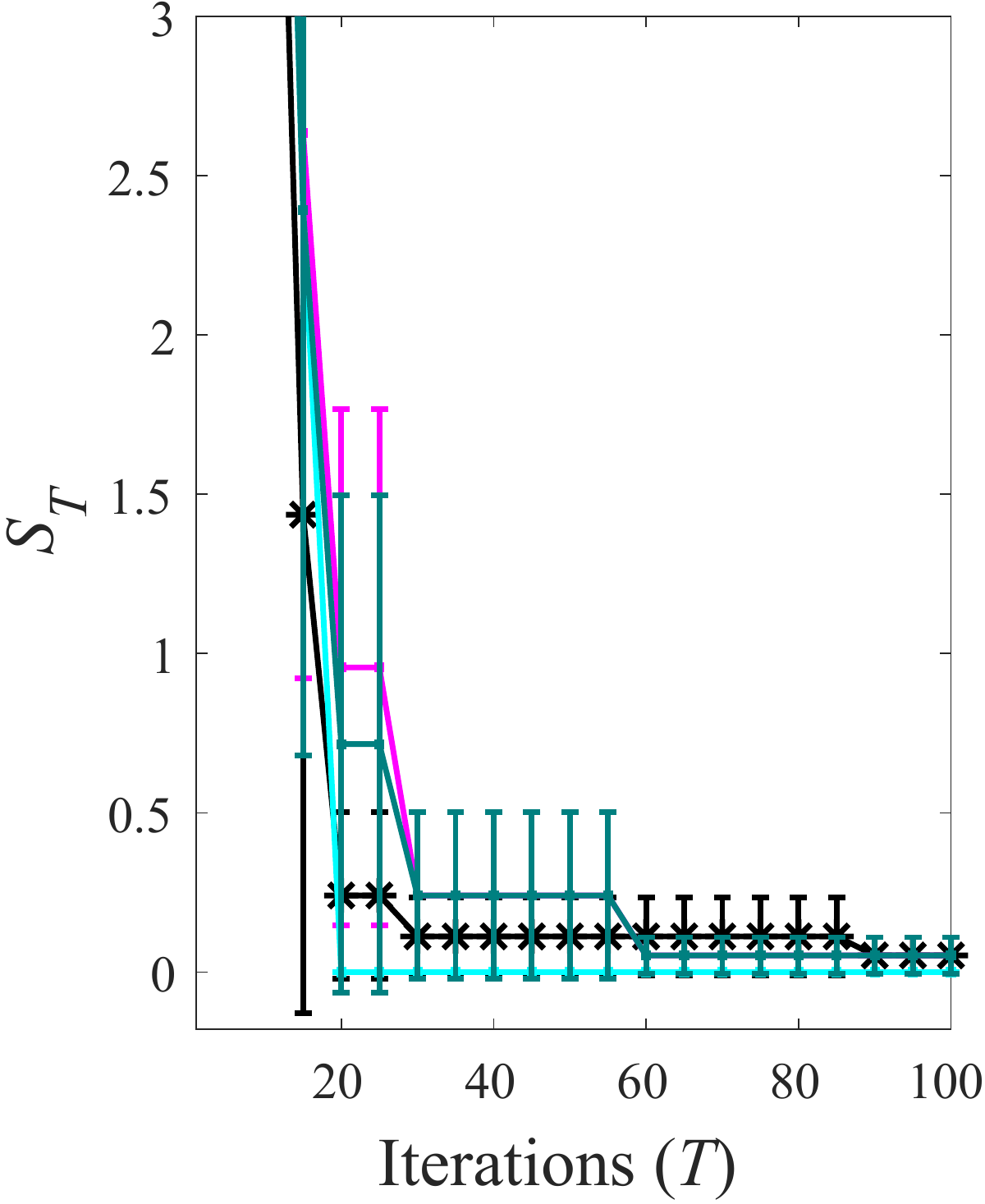}
	\end{minipage}
	\\\vspace{0.2em}
	\begin{minipage}[c]{0.24\linewidth}\centering
		\small(a) \textit{Dropwave}
	\end{minipage}\ \
	\begin{minipage}[c]{0.24\linewidth}\centering
		\small(b) \textit{Griewank}
	\end{minipage}\ \
	\begin{minipage}[c]{0.24\linewidth}\centering
		\small(c) \textit{Hart6}
	\end{minipage}
	\begin{minipage}[c]{0.24\linewidth}\centering
		\small(d) \textit{Rastrigin}
	\end{minipage}
	\caption{The results (mean$\pm$(1/4)std.) of PI-PP and PI on synthetic benchmark functions. $S_T$: the smaller, the better.}\label{fig-PI-benchmark}
\end{figure*}

\begin{figure*}[t!]\centering
\begin{minipage}[c]{0.45\linewidth}\centering
	\includegraphics[width=1\linewidth]{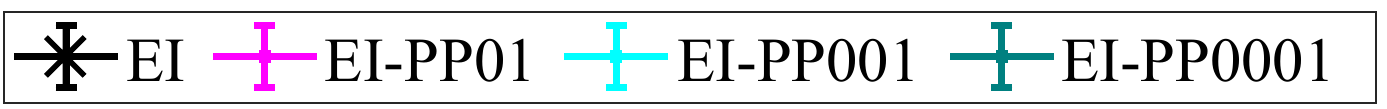}
\end{minipage}
\\\vspace{0.2em}
\begin{minipage}[c]{0.24\linewidth}\centering
	\includegraphics[width=1\linewidth]{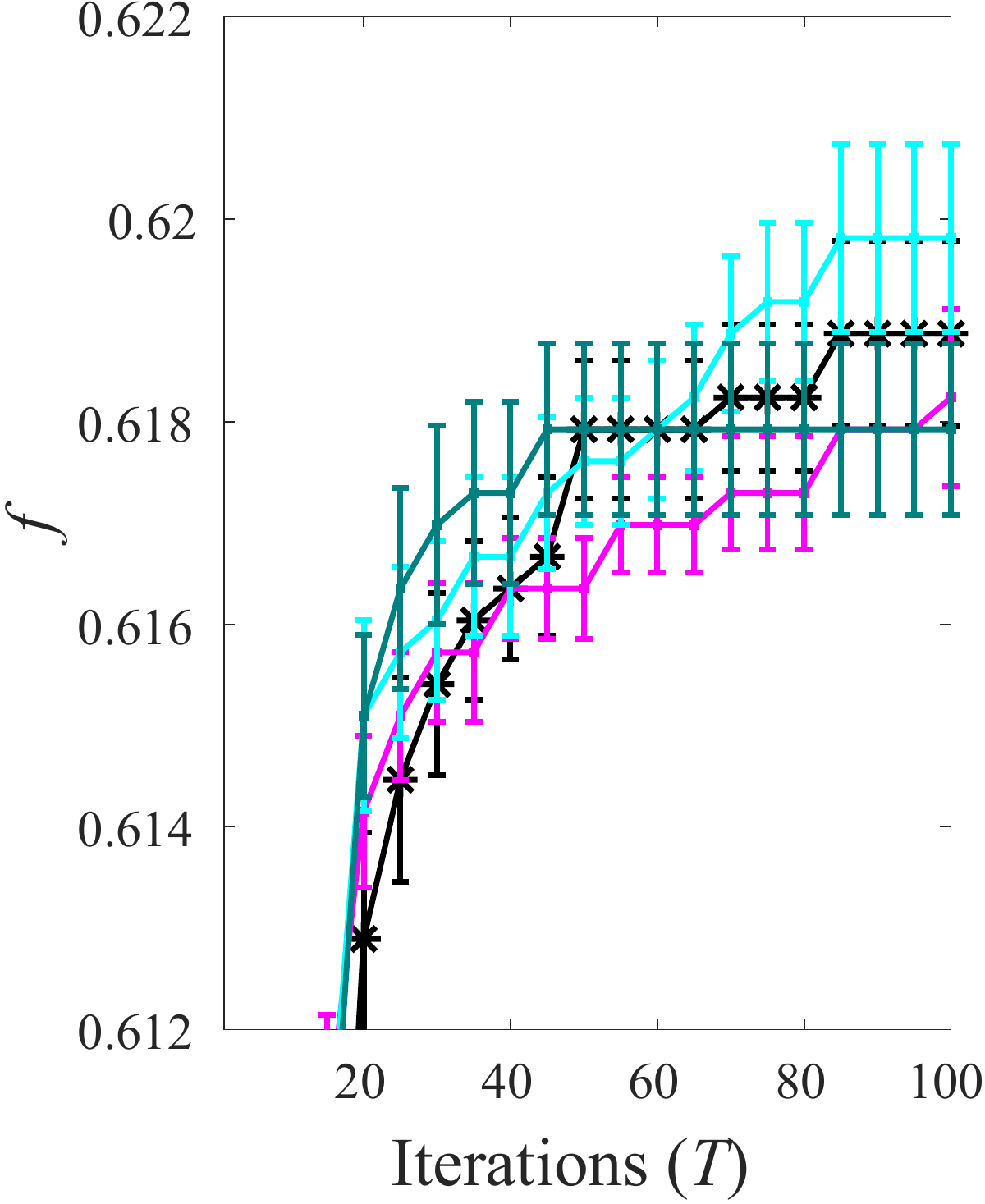}
\end{minipage}\ \
\begin{minipage}[c]{0.24\linewidth}\centering
	\includegraphics[width=1\linewidth]{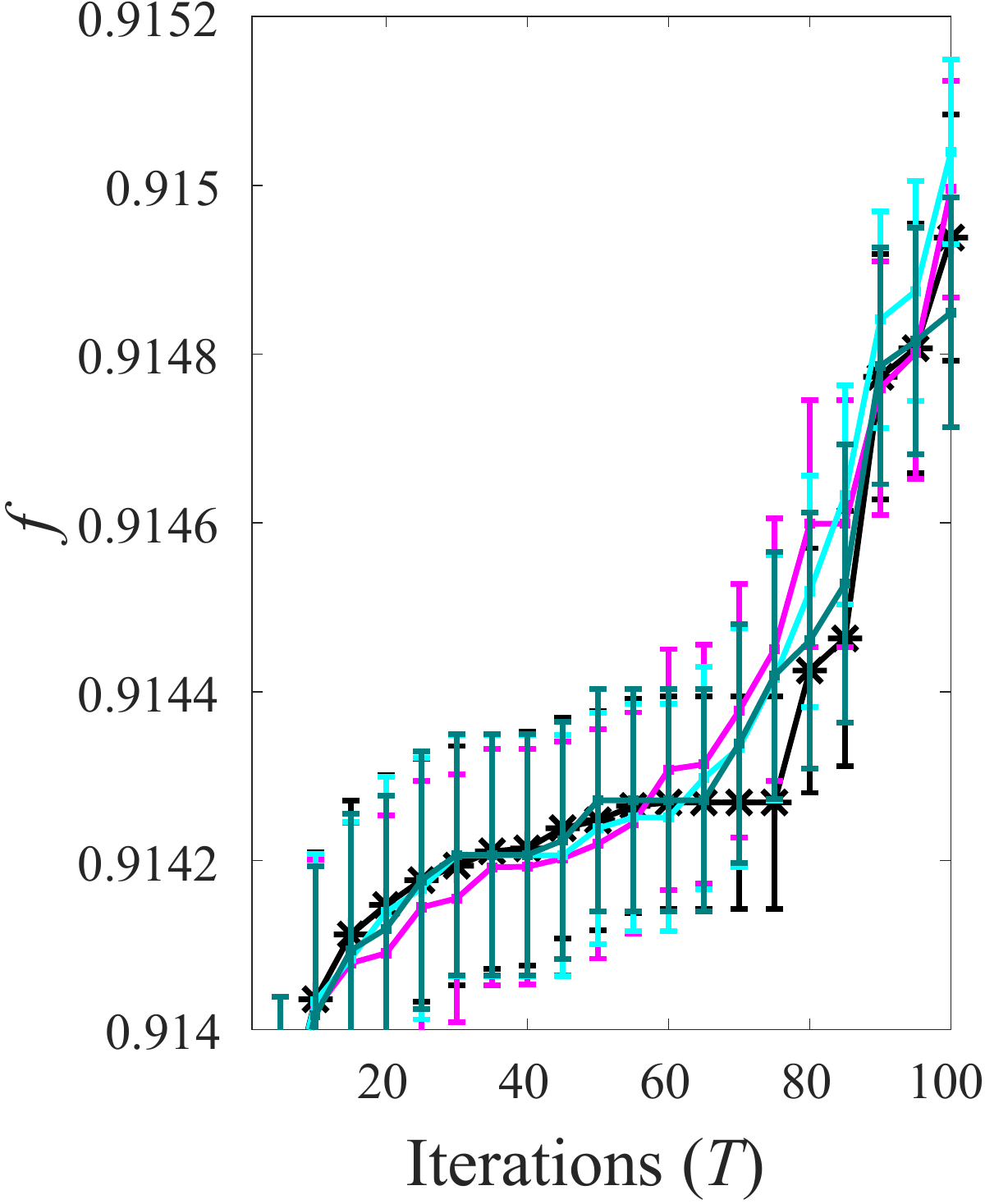}
\end{minipage}\ \
\begin{minipage}[c]{0.24\linewidth}\centering
	\includegraphics[width=1\linewidth]{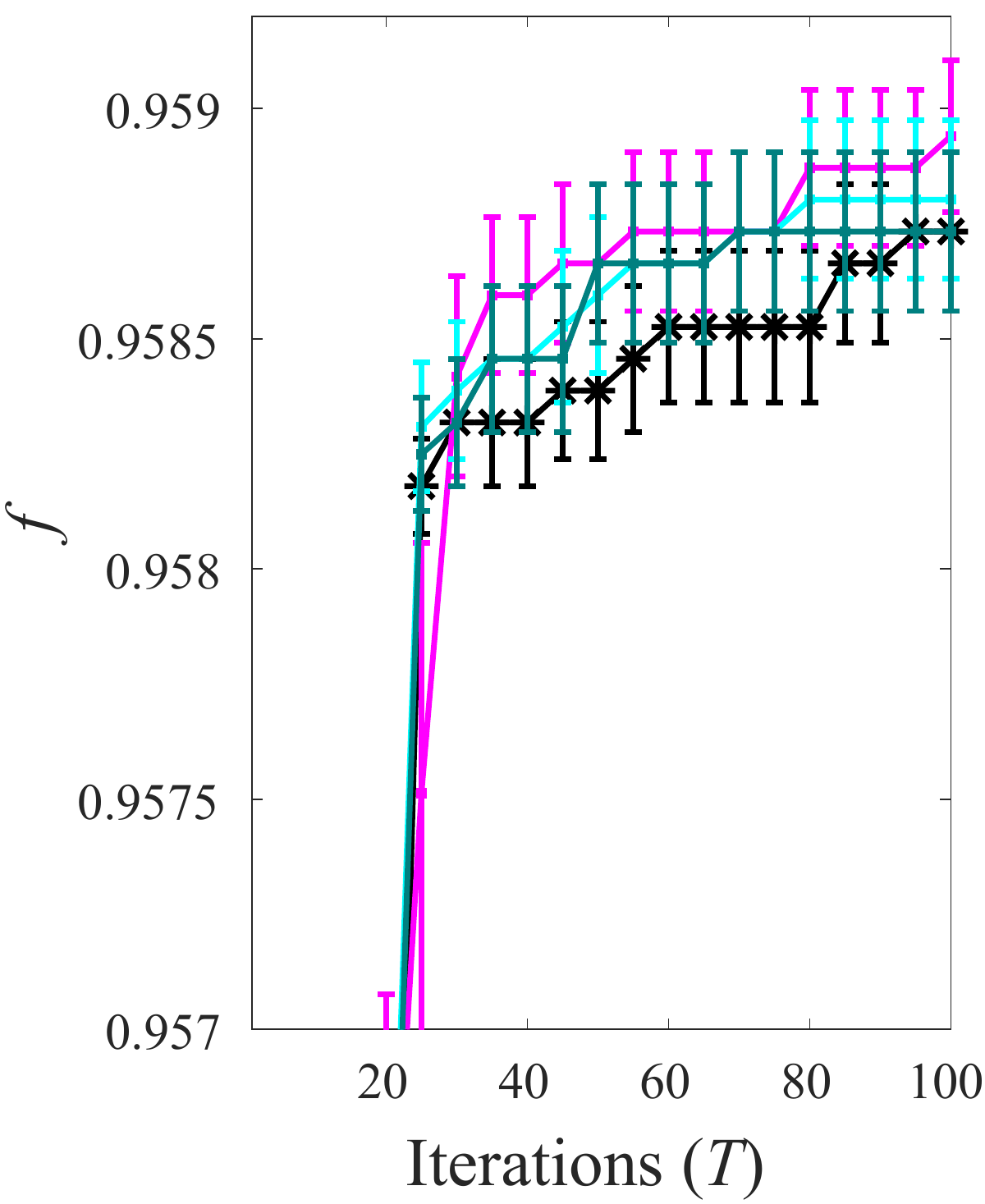}
\end{minipage}
\begin{minipage}[c]{0.24\linewidth}\centering
	\includegraphics[width=1\linewidth]{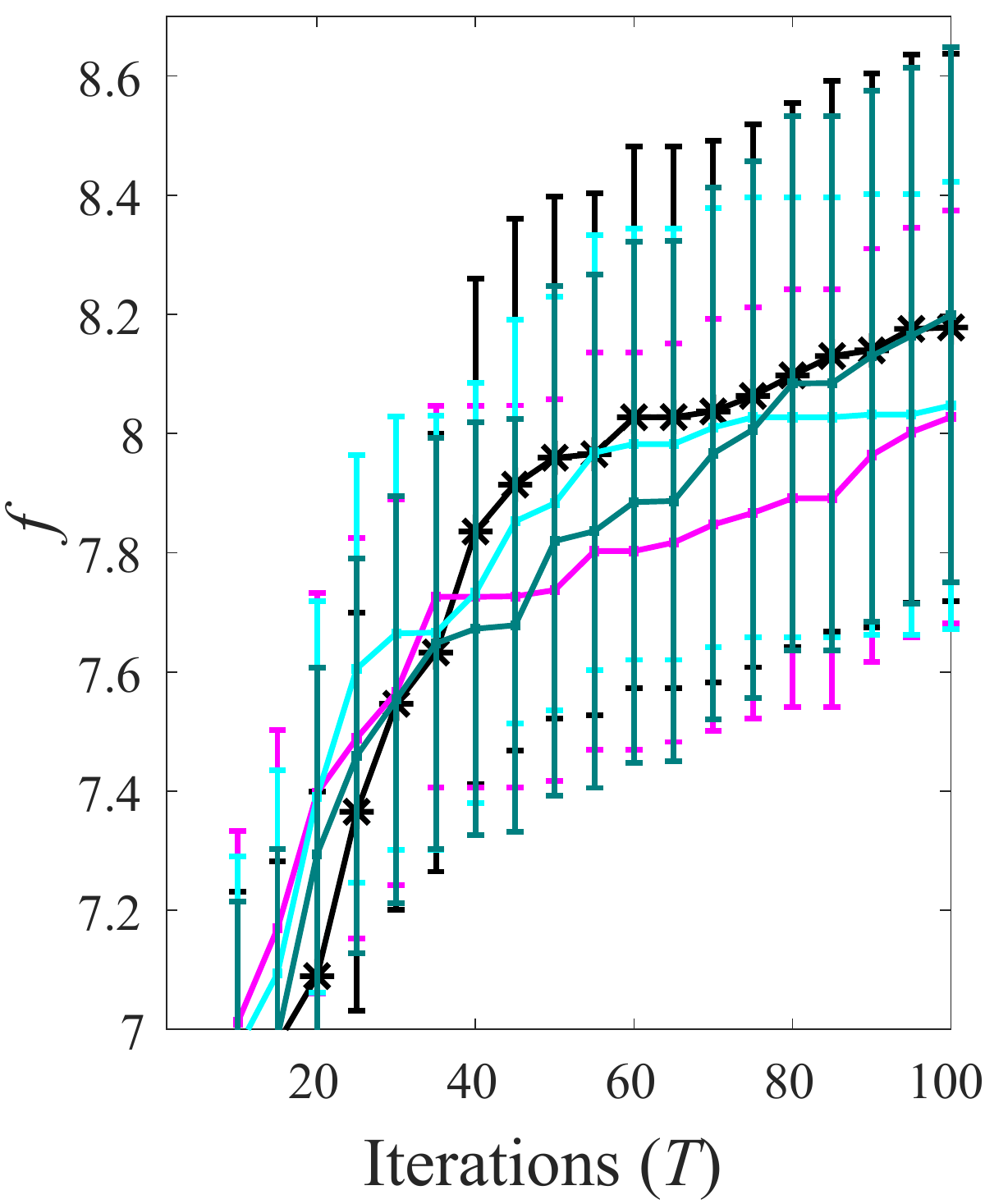}
\end{minipage}\ \
\\\vspace{0.2em}
\begin{minipage}[c]{0.24\linewidth}\centering
	\small(a) \textit{SVM$\_$wine}
\end{minipage}\ \
\begin{minipage}[c]{0.24\linewidth}\centering
	\small(b) \textit{NN$\_$wine}
\end{minipage}\ \
\begin{minipage}[c]{0.24\linewidth}\centering
	\small(c) \textit{NN$\_$cancer}
\end{minipage}
\begin{minipage}[c]{0.24\linewidth}\centering
	\small(d) \textit{NN$\_$housing}
\end{minipage}\ \
	\caption{The results (mean$\pm$(1/4)std.) of EI-PP and EI on real-world optimization problems. $f$: the larger, the better.}\label{fig-EI-real-world}
\end{figure*}

\begin{figure*}[t!]\centering
	\begin{minipage}[c]{0.45\linewidth}\centering
		\includegraphics[width=1\linewidth]{EI_legend}
	\end{minipage}
	\\\vspace{0.2em}
	\begin{minipage}[c]{0.24\linewidth}\centering
		\includegraphics[width=1\linewidth]{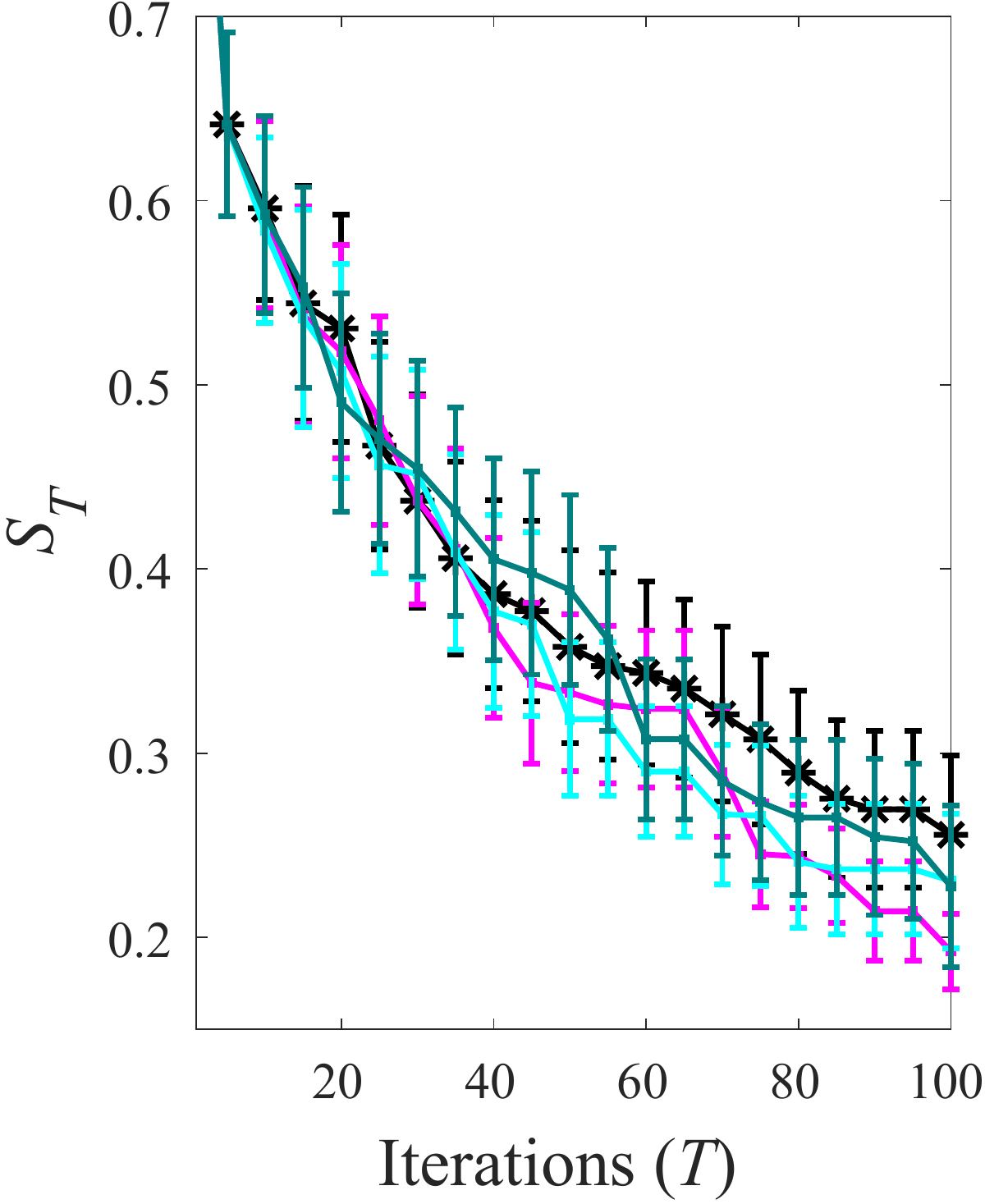}
	\end{minipage}\ \
	\begin{minipage}[c]{0.24\linewidth}\centering
		\includegraphics[width=1\linewidth]{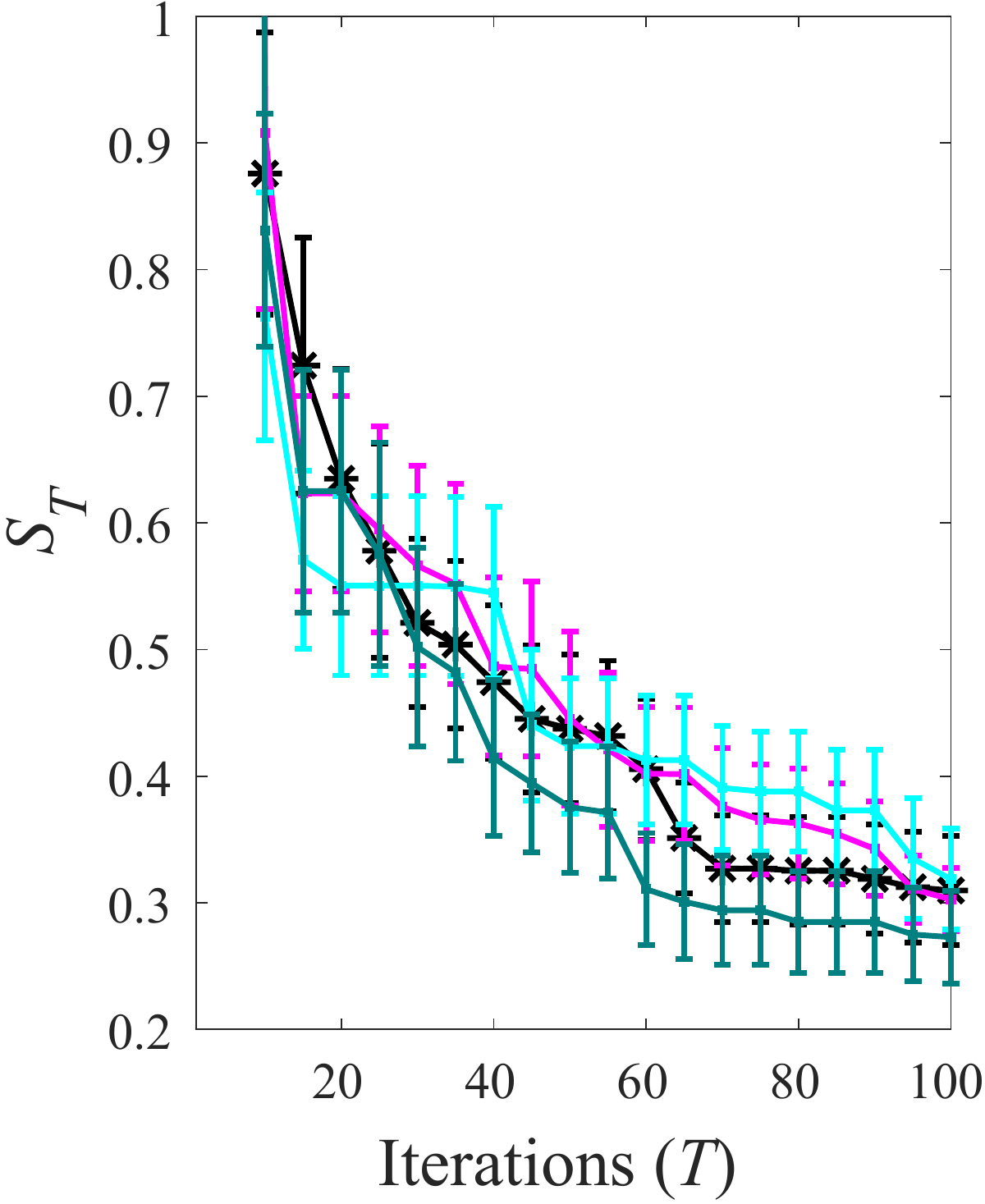}
	\end{minipage}\ \
	\begin{minipage}[c]{0.24\linewidth}\centering
		\includegraphics[width=1\linewidth]{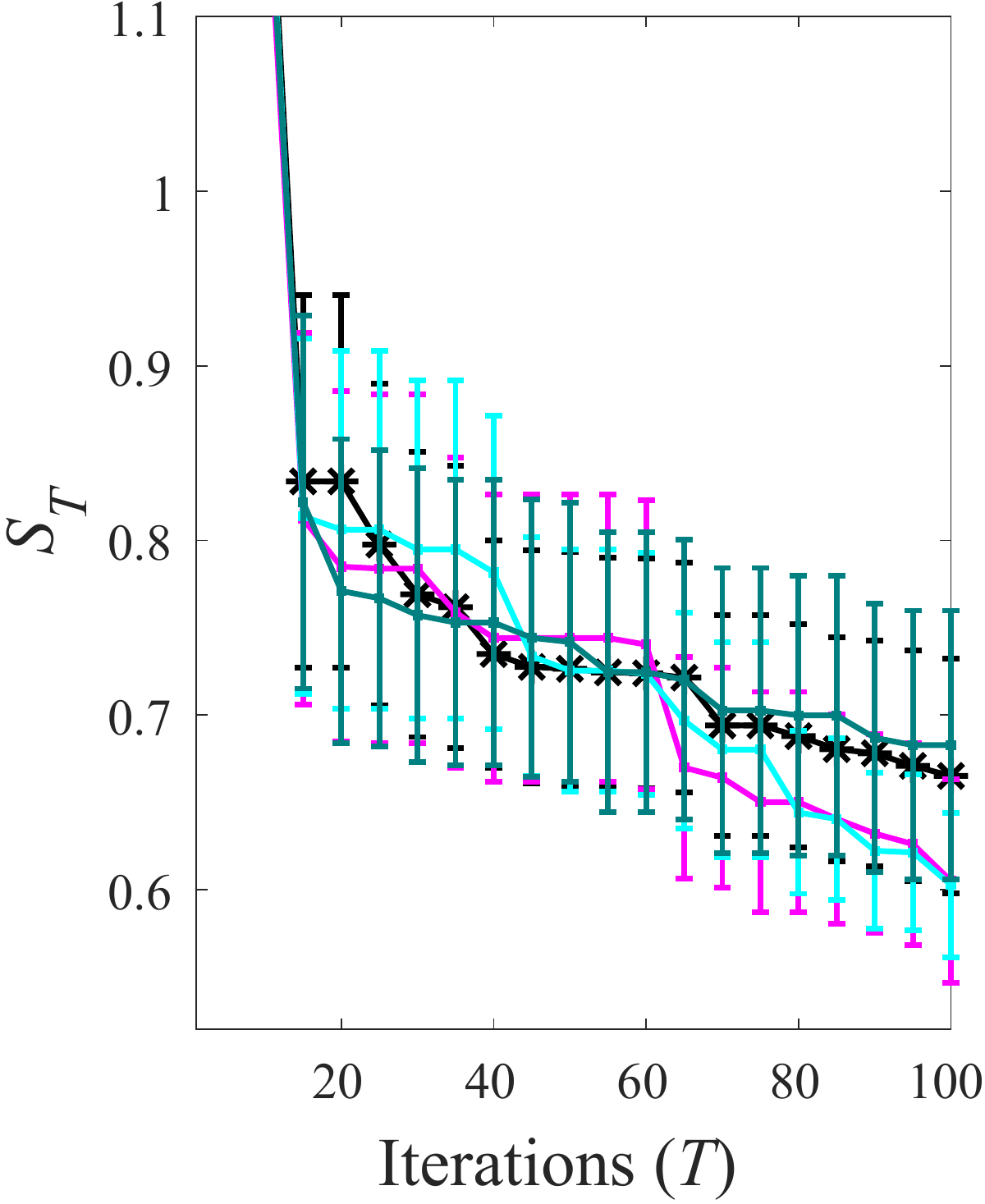}
	\end{minipage}
	\begin{minipage}[c]{0.24\linewidth}\centering
		\includegraphics[width=1\linewidth]{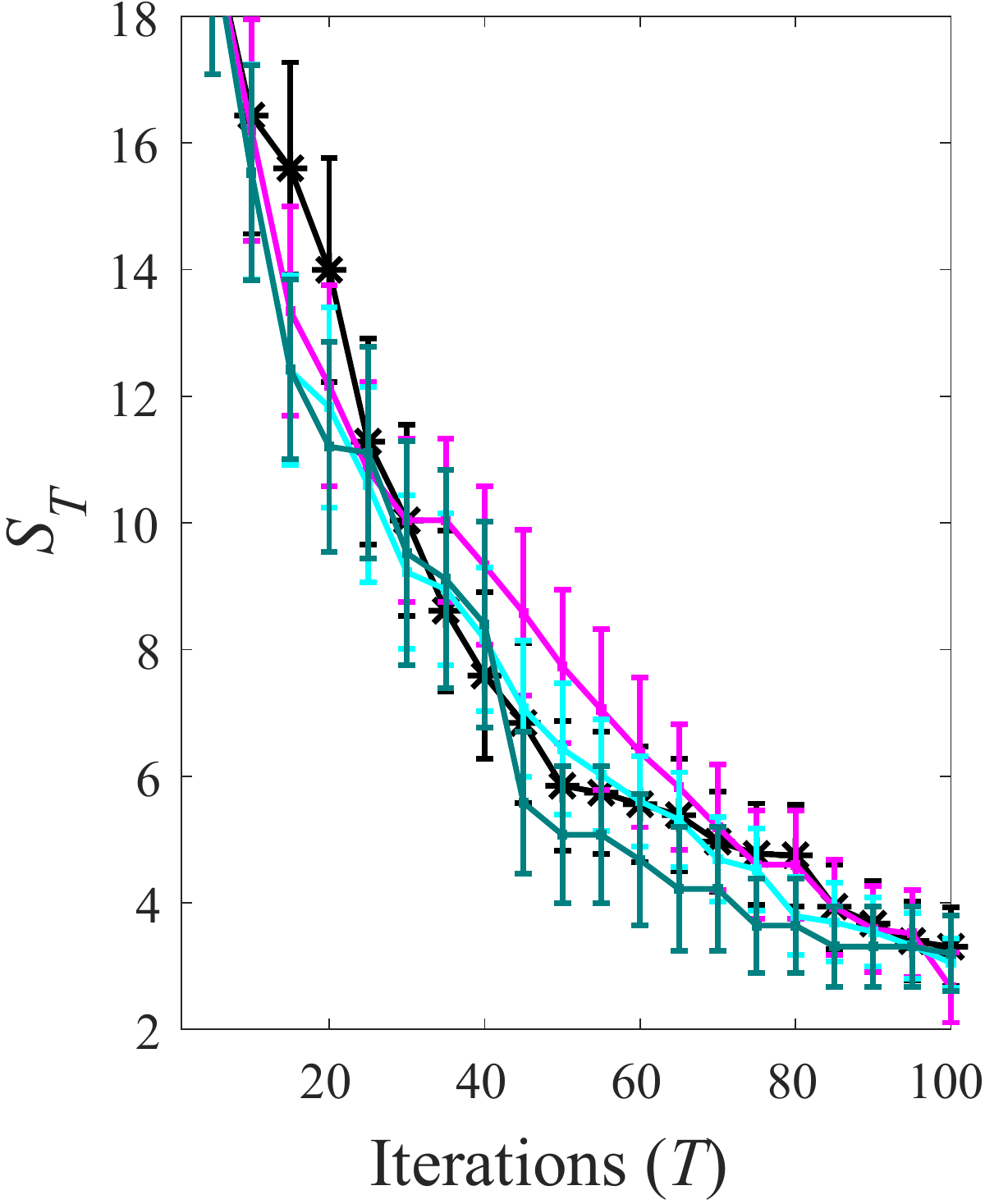}
	\end{minipage}
	\\\vspace{0.2em}
	\begin{minipage}[c]{0.24\linewidth}\centering
		\small(a) \textit{Dropwave}
	\end{minipage}\ \
	\begin{minipage}[c]{0.24\linewidth}\centering
		\small(b) \textit{Griewank}
	\end{minipage}\ \
	\begin{minipage}[c]{0.24\linewidth}\centering
		\small(c) \textit{Hart6}
	\end{minipage}
	\begin{minipage}[c]{0.24\linewidth}\centering
		\small(d) \textit{Rastrigin}
	\end{minipage}
	\caption{The results (mean$\pm$(1/4)std.) of EI-PP and EI on synthetic benchmark functions. $S_T$: the smaller, the better.}\label{fig-EI-benchmark}
\end{figure*}

To examine the robustness of BO-PP against kernels, we use the Gaussian kernel with hyper-parameters tuned by MLE. We compare UCB-PP with UCB on real-world problems, and the results in Figure~\ref{fig-UCB-SE-real-world} show that UCB-PP can always be better except UCB-PP01 and UCB-PP001 on \textit{SVM$\_$wine}.

\begin{figure*}[t!]\centering
\begin{minipage}[c]{0.5\linewidth}\centering
	\includegraphics[width=1\linewidth]{UCB_legend}
\end{minipage}
\\\vspace{0.2em}
\begin{minipage}[c]{0.24\linewidth}\centering
	\includegraphics[width=1\linewidth]{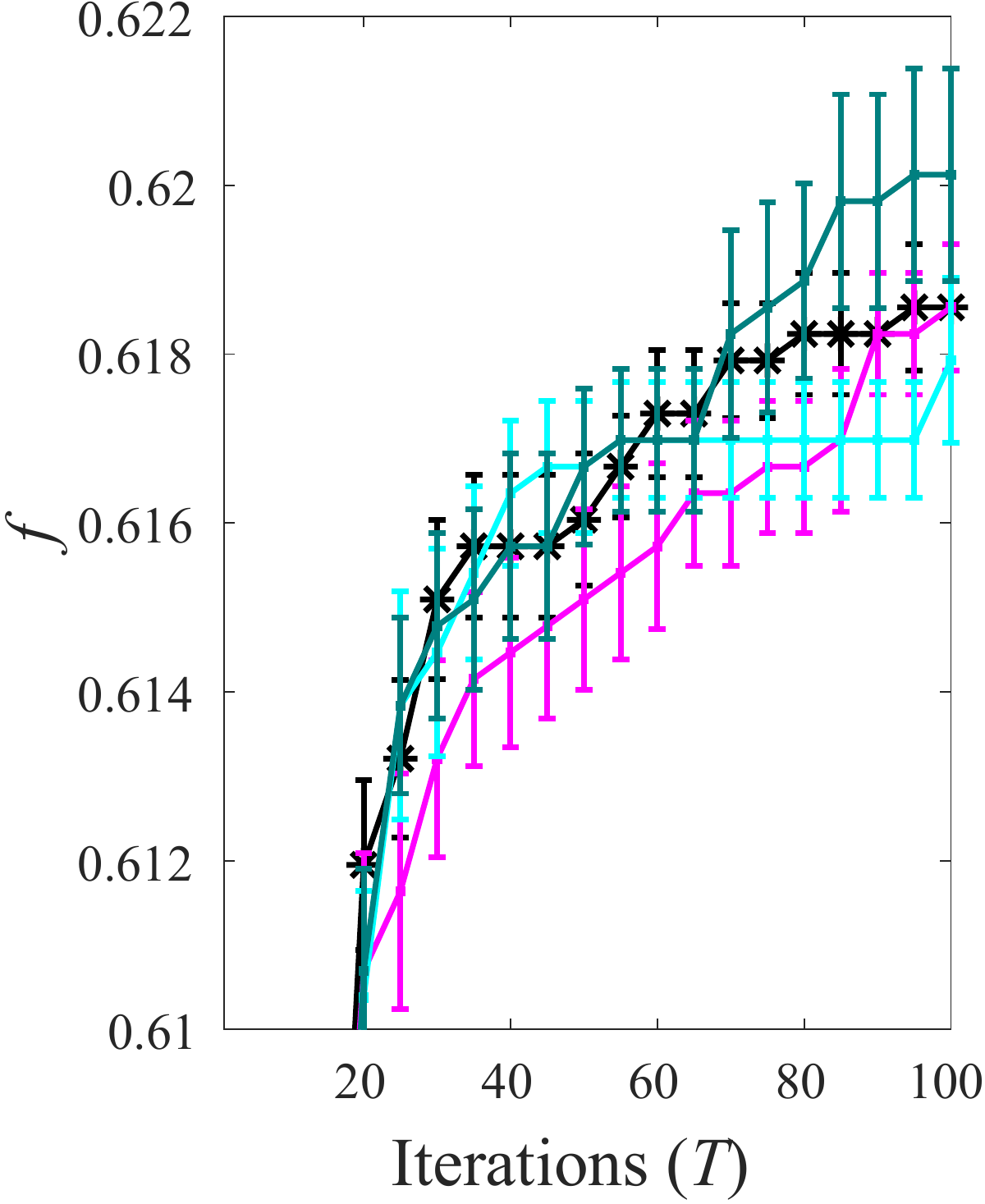}
\end{minipage}\ \
\begin{minipage}[c]{0.24\linewidth}\centering
	\includegraphics[width=1\linewidth]{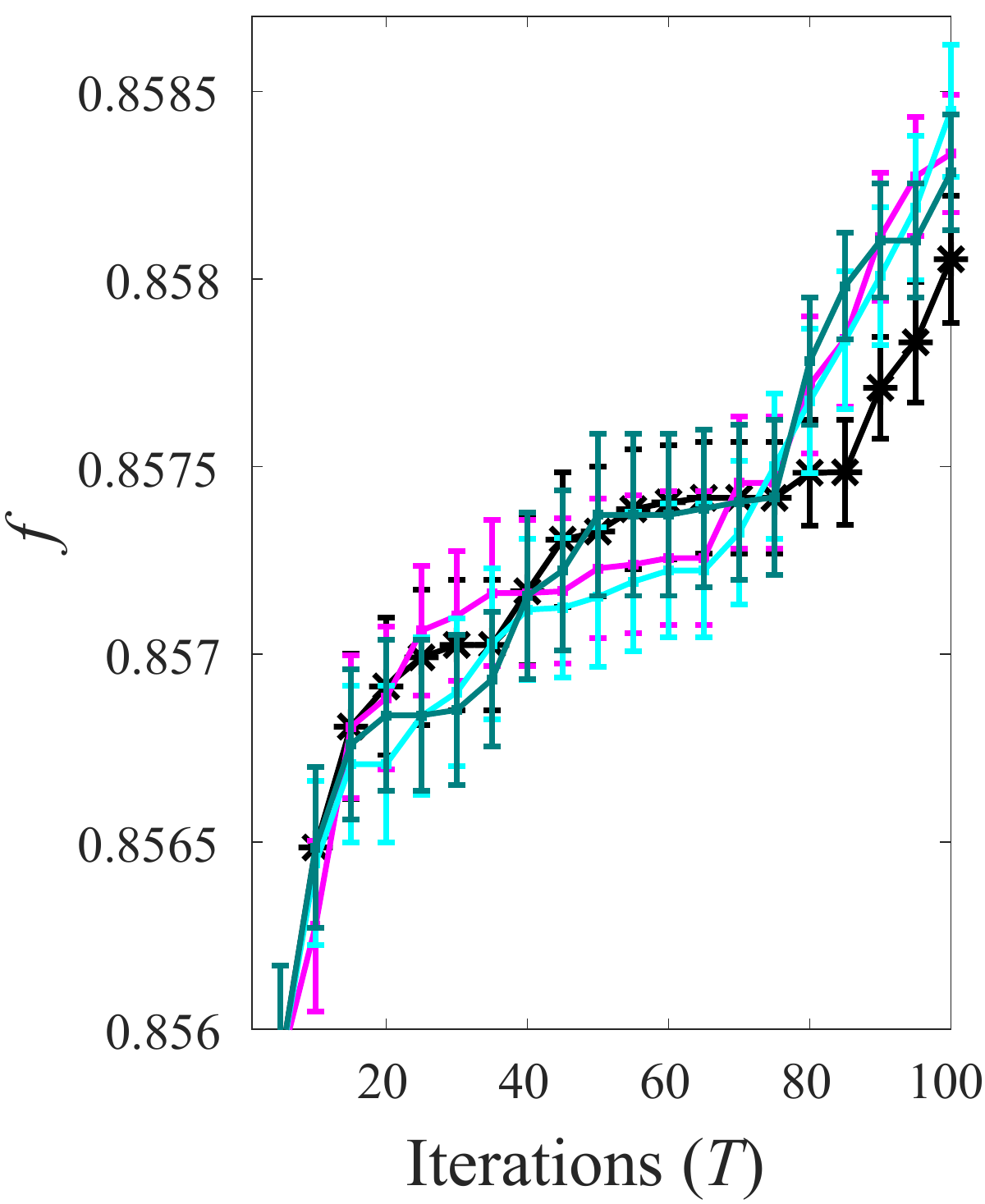}
\end{minipage}\ \
\begin{minipage}[c]{0.24\linewidth}\centering
	\includegraphics[width=1\linewidth]{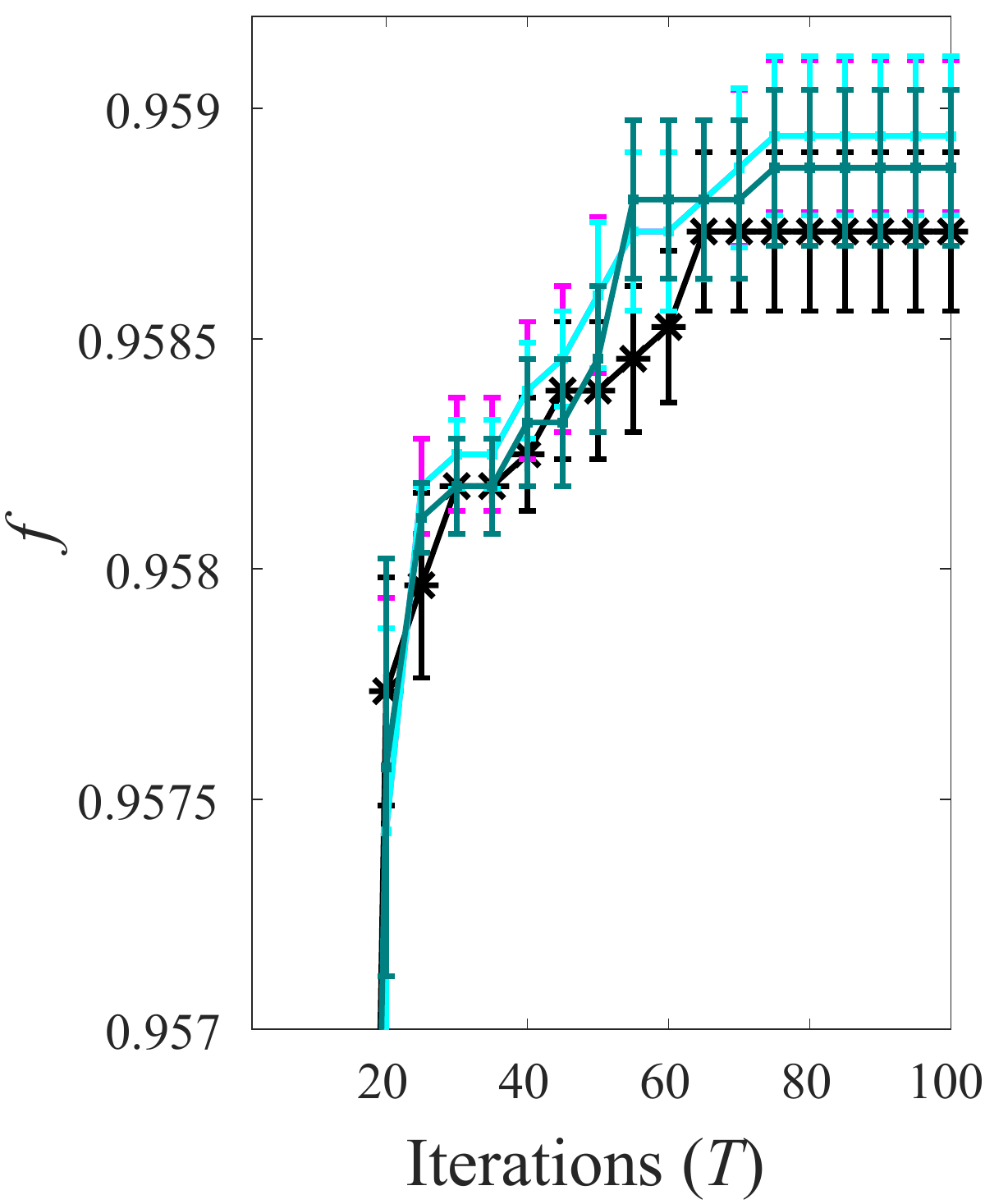}
\end{minipage}
\begin{minipage}[c]{0.24\linewidth}\centering
	\includegraphics[width=1\linewidth]{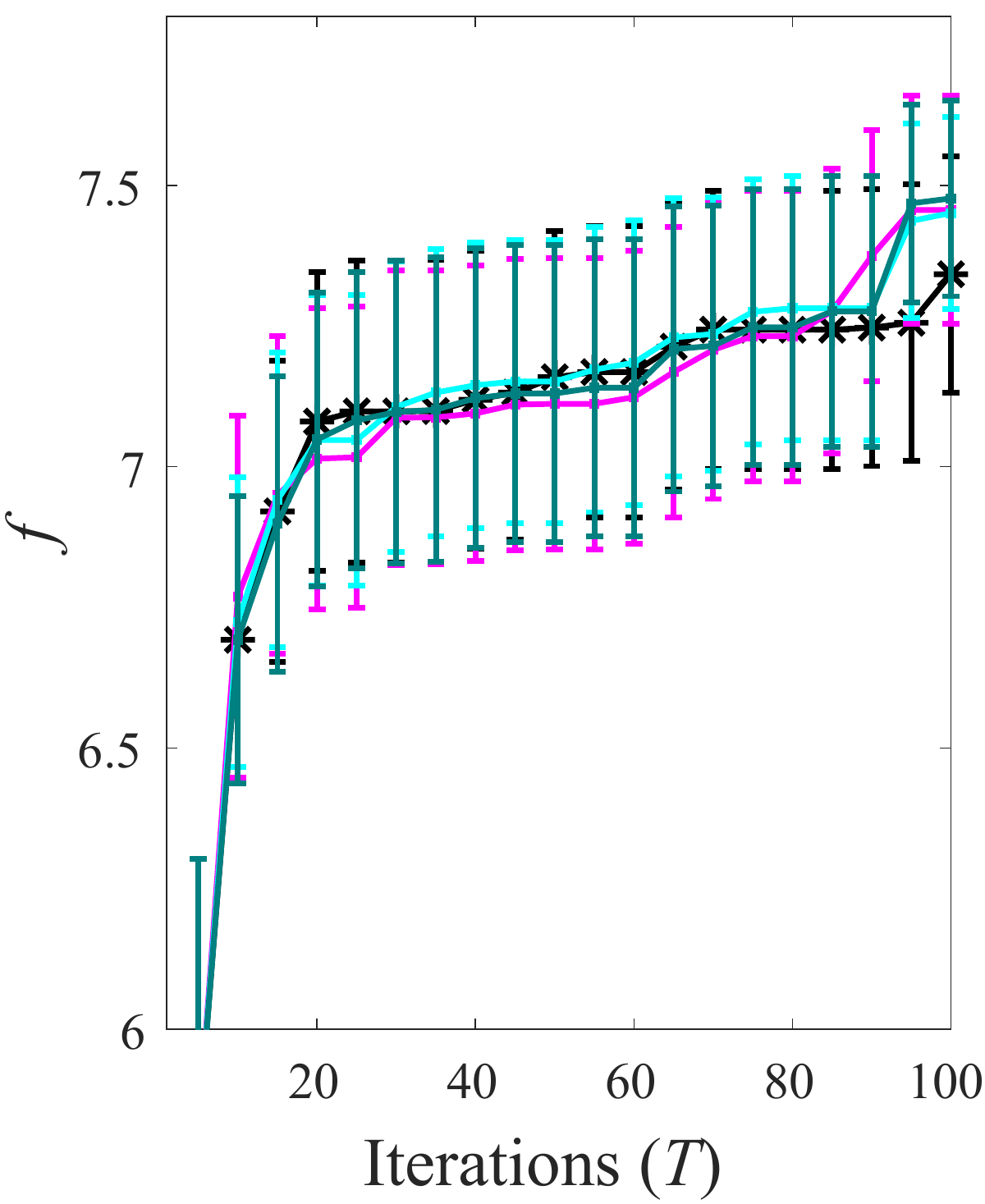}
\end{minipage}\ \
\\\vspace{0.2em}
\begin{minipage}[c]{0.24\linewidth}\centering
	\small(a) \textit{SVM$\_$wine}
\end{minipage}\ \
\begin{minipage}[c]{0.24\linewidth}\centering
	\small(b) \textit{NN$\_$wine}
\end{minipage}\ \
\begin{minipage}[c]{0.24\linewidth}\centering
	\small(c) \textit{NN$\_$cancer}
\end{minipage}
\begin{minipage}[c]{0.24\linewidth}\centering
	\small(d) \textit{NN$\_$housing}
\end{minipage}
\caption{The results (mean$\pm$(1/4)std.) of UCB-PP and UCB with the Gaussian kernel on real-world optimization problems. $f$: the larger, the better.}\label{fig-UCB-SE-real-world}
\end{figure*}

\section{Conclusion}

In this paper, we propose a general framework BO-PP by generating pseudo-points to improve the GP model of BO. BO-PP can be implemented with any acquisition function. Equipped with UCB, we prove that the cumulative regret of BO-PP can be well bounded. Experiments with UCB, PI and EI on synthetic as well as real-world optimization problems show the excellent performance of BO-PP. It is expected that the generation of pseudo-points can be helpful for more BO algorithms.

\bibliographystyle{ACM-Reference-Format}
\bibliography{telo-bopp}

\end{document}